\documentclass[review]{elsarticle}

\usepackage{lineno,hyperref}
\modulolinenumbers[5]

\journal{Journal of \LaTeX\ Templates}
\usepackage[utf8x]{inputenc}
\DeclareMathAlphabet\mbi{OML}{cmm}{b}{it}
\DeclareSymbolFont{boldsymbols}{OMS}{cmsy}{b}{n}
\DeclareSymbolFontAlphabet{\mathbfcal}{boldsymbols}

\usepackage{geometry}
\usepackage{amsthm}
\usepackage{amsmath}
\usepackage{amsbsy}
\usepackage{amsfonts}
\usepackage{dsfont}
\usepackage{mathrsfs}
\everymath{\displaystyle}
\newtheorem{thm}{Theorem}[section]

\theoremstyle{definition}

\theoremstyle{Remarque}

\theoremstyle{plain}
\newtheorem{lem}[thm]{Lemma}

\usepackage{graphicx}
\usepackage{rotating}
\usepackage{marvosym}
\usepackage{textcomp}
\usepackage {setspace}
\usepackage{fancyhdr}
\usepackage{hyphenat}
\usepackage{multirow}
\usepackage{pdfpages}
\usepackage{pdftexcmds}
\usepackage{subfigure}
\usepackage{booktabs,caption,fixltx2e}
\usepackage[flushleft]{threeparttable}
\usepackage{latexsym}
\usepackage{lmodern}
\usepackage{a4wide}
\usepackage{lscape}
\usepackage{rotating}
\usepackage[flushleft]{threeparttable}
\usepackage{algorithm,algorithmic} 

\theoremstyle{definition}
\newtheorem{definition}{Definition}[section]
\newtheorem{proposition}{Proposition}[section]
\newtheorem{theorem}{Theorem}[section]
\newtheorem{remark}{Remark}[section]

\DeclareMathOperator*{\argmin}{arg\,min}

\DeclareMathOperator*{\var}{Var}
\DeclareMathOperator*{\E}{\mathbb{E}}
\DeclareMathOperator*{\Z}{\mathbb{Z}}
\newcommand{\R}{\mathbb{R}}
\biboptions{authoryear}
\usepackage{filecontents}

\begin{document}

\begin{frontmatter}

\title{Kalman Recursions Aggregated Online}
% Consistent estimates in the multivariate linear mixed-effects model

%\author{Eric Houngla Adjakossa\fnref{myfootnote}}
%\author{Gregory Nuel\fnref{myfootnote}}

\author[mymainaddress]{Eric Adjakossa\corref{correspondingauthor}}
\cortext[correspondingauthor]{Corresponding author}
\ead{ericadjakossah@gmail.com}

\author[mysecondaryaddress]{Yannig Goude}
\ead{yannig.goude@edf.fr}

\author[mymainaddress]{Olivier Wintenberger}
\ead{olivier.wintenberger@upmc.fr}

\address[mymainaddress]{Sorbonne Universit\'e, Laboratoire de Probabilit\'es, Statistique et Mod\'elisation (LPSM, UMR 8001), 4 place Jussieu, 75005 Paris, France}
\address[mysecondaryaddress]{EDF Lab, 7 Boulevard Gaspard Monge, 91120 Palaiseau}

\begin{abstract}
In this article, we aim at improving the prediction of expert aggregation by using the underlying properties of the models that provide expert predictions. We restrict ourselves to the case where expert predictions come from Kalman recursions, fitting state-space models. By using exponential weights, we construct different algorithms of Kalman recursions Aggregated Online (KAO) that compete with the best expert or the best convex combination of experts in a more or less adaptive way. We improve the existing results on expert aggregation literature when the experts are Kalman recursions by taking advantage of the second-order properties of the Kalman recursions. We apply our approach to Kalman recursions and extend it to the general adversarial expert setting by state-space modeling the errors of the experts. We apply these new algorithms to a real dataset of electricity consumption and show how it can improve forecast performances comparing to other exponentially weighted average procedures.
\end{abstract}


\end{frontmatter}

%\linenumbers

\section{Introduction}

The aim of this paper is to aggregate Kalman recursions in an online setting in order to increase the accuracy of the prediction. We observe $(y_t)$ sequentially through time $t\ge 1$ and   predictions $\hat{y}_t^{(m)}$, $1\le m\le M$,  issued from  Kalman recursions  at time  $t\ge 1$.  Here $M\ge 0$ denotes the number of different Kalman recursions used as experts. The Kalman recursions are imbedded into a state-space model (see Section \ref{sec:model} for a formal definition). We introduce different Kalman recursions Aggregated Online (KAO)  procedures that compute  recursively weights $\rho_t^{(m)}$, $t\ge 1$, $1\le m\le M$. We provide theoretical guarantees on the average prediction $\hat y_t=\sum_{m=1}^M\rho_t^{(m)}\hat{y}_t^{(m)}$. \\

We obtain bounds on the regret of KAO algorithms that are similar to the ones encountered in the literature.  The book of reference on aggregation is undoubtedly the book from~\cite{cesa2006prediction} and we refer to it for classical regret aggregation bounds. The novelty of our approach is to derive regret bounds directly on the cumulative quadratic predictive risk as defined by \cite{wintenberger2017optimal}. The predictive risk or risk of prediction of a predictor $\hat y\in \mathcal F_{t-1}$ is defined as 
 \begin{equation}\label{eq:loss_function}
L_t(\hat{y})=\E\left[(\hat{y}-y_t)^2\mid \mathcal F_{t-1}\right]\,,\qquad a.s.,\qquad t\ge1\,,
 \end{equation}
 where $(\mathcal F_{t})$ is the natural filtration of the past response variables $\sigma(y_s;\, 0\le s\le t)= \mathcal F_t$, $t\ge0$. The risk of prediction arises naturally when dealing with Kalman  recursions. Indeed, Kalman recursions are online algorithms that provide the best linear predictions in gaussian state-space models. We refer to the classical monograph from \cite{durbin2012time} for details. The cumulative predictive risk of a recursive algorithm predicting $\hat y_t$ at each time $t\ge 1$ is  the sum $\sum_{t=1}^TL_t(\hat{y}_t)$ up to the horizon $T\ge 1$. Our regret bounds on KAO algorithm predictions $(\hat y_t)$ are a.s. deterministic  bound  called respectively model selection regret or aggregation regret  and defined as
\begin{eqnarray}
\label{regret:S}R_T^S(m)&\ge&\sum_{t=1}^TL_t(\hat{y}_t) - L_t(\hat y_t^{(m)})\qquad 1\le m\le M\,,\\
\label{regret:A}R_T^A(\pi)&\ge&\sum_{t=1}^TL_t(\hat{y}_t) - L_t\left(\sum_{m=1}^M\pi^{(m)}\hat{y}_s^{(m)}\right)\,,
\end{eqnarray}
for any  vector of weights $\pi:=(\pi^{(m)})_{1\le m\le M}$ in the simplex. We suppress the dependence in $m$ and $\pi$ in $R_t^S$ and $R_t^A$ when the regret bounds are uniform in $m$ and $\pi$, respectively.\\

Regret bounds on the cumulative predictive risk have attracted some interest since \cite{audibert2010regret} showed that the classical EWA algorithm from \cite{vovk1990aggregating} does not achieve a fast rate model selection regret even in the most favorable iid case. The fast rate  model selection regret was first proved by the BOA algorithm in the iid setting with strongly convex loss in  \cite{wintenberger2017optimal} and then extended to any stochastic setting and exp-concave risk in \cite{gaillard2016sparse}. For adaptative procedures, we improve their optimal regret
$$
R_T^S = O(\log M+\log\log T + x)\,,\qquad T\ge 1
$$
with probability $1-e^{-x}$, $x>0$, to the a.s. bound
 $$
R_T^S = O(\log M+\log \log T)\,\qquad T\ge 1\,.
$$
This optimal regret bound holds when the observations satisfy some unbounded state-space model defined in the next section \ref{sec:model}.   That the observations satisfy such a model is very unlikely in practice; It is the price to pay to guarantee a.s. optimal regret bounds on the cumulative predictive risk \eqref{eq:loss_function} for unbounded responses. Existing regret bounds such as the one of \cite{gaillard2016sparse} requires the boundedness of the response.\\

We present simulations and applications in cases where our assumptions are certainly not satisfied. Our new aggregation procedure improves the state of the art methods in aggregation such as MLPoly of \cite{gaillard2014second}.
\section{Preliminaries}

\subsection{State-space models}\label{sec:model}
Assume that we observe $(y_t,X_t)$ with $y_t\in \R$ the variable of interests and $X_t\in\R^d$ is the predictable  design, i.e. $X_t \in \mathcal F_{t-1}$, $t\ge1$. Notice that the design $X_t$ can be either deterministic  or random. We consider a collection of $M\ge1$ experts $\hat y_t^{(m)}=X_t^\top \hat \theta_t^{(m)}$ issued from Kalman recursions as follows. We denote $\E_t[\cdot ]$ and $\var_t(\cdot)$ the conditional expectation $\E[\cdot \mid \mathcal F_t]$ and variance $\var(\cdot\mid \mathcal F_t)$, respectively, for any $t\ge 0$.

For each $1\le m\le M$, the sequence of experts $\hat y_t^{(m)}=X_t^\top \hat \theta_t^{(m)}$ is  associated to a recursive hidden state model
\begin{equation}\label{eq:stateq}
\theta_{t}^{(m)} = K^{(m)}\theta_{t-1}^{(m)} +z_t^{(m)}\,,\qquad t\ge 1\,,
\end{equation}
where $ K^{(m)}$ is a $d\times d$ matrix, $z_t^{(m)}\sim\mathcal{N}(0,Q^{(m)})$ constitute an iid sequence and $\theta_{0}^{(m)}\in \R^d$ is deterministic. The sequence $(\theta_{t}^{(m)})$ is a Gaussian Markov chain and admits the representation
\begin{equation}\label{eq:markov}
\theta_{t}^{(m)}=\big({K^{(m)}}\big)^t\theta_0^{m}+\sum_{k=0}^{t-1}\big({K^{(m)}}\big)^kz_{t-k}^{(m)}\,,\qquad t\ge 1\,,1\le m\le M\,.
\end{equation}
It converges weakly to a stationary solution if and only if $\rho(K^{(m)})<1$, the spectral radius of $K^{(m)}$ is smaller than one. This assumption is not required in this work. Actually one of the most popular state models is the dynamic setting where one considers random walk under $K^{(m)}=I_d$, the identity matrix of $\R^d$. We notice that as the state model is hidden (latent, not observed), any  assumption on the state recursion, such as the Gaussian assumption, is not restrictive for the observations $(y_t,X_t)$.\\

Our main assumption is the following one.
\begin{enumerate}
\item[\bf (H)] The vectors $(y_t,\theta_t^{(m)})$ constitute a Gaussian sequence,
$$
\E[y_t\,\mid\,\theta_t^{(m)}]=X_t^\top \theta_t^{(m)}\,,\qquad t\ge 1\,,1\le m\le M\,,
$$
and the conditional variance $\sigma^{2(m)}:=\var(y_t\,\mid\,\theta_t^{(m)})>0$ is constant through time and known.
\end{enumerate}
Condition ({\bf H}) have different consequences upon the observations $(y_t,X_t)$. The first obvious one is that $(y_t)$ constitutes a Gaussian sequence. The second one is that $y_t$ satisfies the linear model 
$$
y_t=X_t^\top\theta_t^{(m)}  +\varepsilon_t^{(m)} \,,\qquad t\ge 1\,,1\le m\le M\,.
$$
The gaussian property of the couple $(y_t,\theta_t^{(m)})$  ensures that $\varepsilon_t^{(m)}$ is a gaussian random variable with mean zero and variance $\sigma^{2(m)}=\var(y_t\,\mid\,\theta_t^{(m)})$ independent of $t\ge1$. A direct implication from the expression \eqref{eq:markov} is the following mean-variance identity.
\begin{proposition}\label{prop:mean}
Under Condition {\bf (H)} the following mean-variance identity holds for all $1\le m\le M$ and $t\ge1$:
\begin{align*}
\E[y_t]&=\E[X_t^\top\theta_t^{(m)}]=\E[X_t]^\top \big(K^{(m)}\big)^t\theta_0^{m}\,,\\
\var(y_t)&=\var[X_t^\top\theta_t^{(m)}] +\sigma^{2(m)}\\
&=\E\Big[X_t^\top \sum_{k=0}^{t-1}\big({K^{(m)}}\big)^kQ^{(m)}\big({K^{(m)\top}}\big)^kX_t\Big]+\sigma^{2(m)}\,.
\end{align*}
\end{proposition}
The static state-space model setting corresponds to the case where $Q^{(m)}=0$  so that  $\var(y_t)=\sigma^{2(m)}=\sigma^2$, $1\le m\le M$, $t\ge 1$.

\subsection{The Kalman recursion}\label{subsec:kalman_rec}
For the sake of completeness, we recall the Kalman recursion associated with the $m$-th state-space model in Algorithm \ref{alg:Kalman}. For details on the Kalman recursion we refer to the monograph from \cite{durbin2012time}.
\begin{algorithm}[!t]
\caption{Kalman recursion in the $m$-th state-space model}
    \label{alg:Kalman}
\begin{flushleft}
    {\bfseries Parameters:} The matrices $Q^{(m)}$  and $K^{(m)}$.\\
     {\bfseries Initialization:} The matrix $P_0^{(m)}$  and the vector $\hat \theta_0^{(m)}$. \\
    {\bfseries Recursion:}     For each iteration $t=1,\dots,T$ do:
\end{flushleft}
\begin{eqnarray*}
\widehat{\theta}_{t+1}^{(m)} &=&K^{(m)}\left(\widehat{\theta}_t^{(m)}+\dfrac1{X_t^\top P_t^{(m)}X_t+1}P_t^{(m)}X_t(y_t-\hat y_t^{(m)})\right)\,,\\
P_{t+1}^{(m)}&=&K^{(m)}\left(P_t^{(m)}-\dfrac1{X_t^\top P_t^{(m)}X_t+1}P_t^{(m)}X_tX_t^\top {P_t^{(m)}}^\top\right){K^{(m)}}^{\top}+Q^{(m)}\\
\hat y_{t+1}^{(m)}&=&X_{t+1}^\top \widehat{\theta}_{t+1}^{(m)} \,.
\end{eqnarray*}
\end{algorithm}
We notice that the Kalman recursion does not require any inversion of matrices. Each iteration has thus  a $O(d^2)$ computational cost. Moreover it does not require the knowledge of the parameters $\sigma^{2(m)}>0$. In addition, in many cases $X_t$ is in fact a vector of size $d=\sum_{m=1}^Md_m$ that stacks $M$ vectors $X_t^{(m)}\in \R^{d_m}$. In this case one considers $d_m$ sparse vectors $\theta_t^{(m)}$  and identifies them with their non-null components $\theta_t^{(m)}\in\R^{d_m}$. Then the space equation is written as 
$$
y_t={X_t^{(m)}}^\top\theta_t^{(m)}  +\varepsilon_t^{(m)} \,,\qquad t\ge 1\,,1\le m\le M\,,
$$
using similarly the notation $\varepsilon_t^{(m)}\in \R^{d_m}$. Doing so, each Kalman recursion  holds in a state space of dimension $d^{(m)}<d$, lowering the computational cost of each iteration to $O(d_m^2)$.
\\

In the static case when $K^{(m)}=I$ and $Q^{(m)}$ is the null matrix then, using the Shermann-Morrisson formula, we have the alternative recursion for $R_t^{(m)}$, the inverse of $P_t^{(m)}$,
$$
R_{t+1}^{(m)}=R_{t}^{(m)}+X_tX_t^\top\,.
$$
When $P_0^{(m)}$ is taken equals to $1/\lambda^{(m)}I_{d}$, for some $\lambda^{(m)}>0$, then the estimator computed recursively using the Kalman recursion coincides with the Ridge estimator
$$
\hat \theta_t^{(m)}= \arg\min_{\theta\in\R^d}\left\{\sum_{s=1}^t \big(y_s-X_s^\top \theta\big)^2+ \frac{\lambda^{(m)}}2\|\theta-\hat \theta_0^{(m)}\|_2^2\right\}\,.
$$
This equivalence has been first established by \cite{diderrich1985kalman}.\\

Notice that there is no assumption on the dependence among the Kalman recursions. Otherwise, it was possible to consider a Kalman recursion over the stack of the models in a $dM$ dimensional state-space model. However, this approach is not practical when the computational cost $O((dM)^2)$ of the complete Kalman recursion is prohibitive because $M$ is too large. The aim of this work is to show that this ideal procedure, that is  uncertain in practice when the dependence among the recursions has to be estimated, can be easily overcome by a simple aggregation procedure over $M$  Kalman recursions.

\subsection{Examples}
We provide some classical examples of state-space models satisfying {\bf (H)}:
\begin{enumerate}
\item \underline{The static iid setting:} this degenerate case coincides with the usual gaussian linear model for fixed or random (iid) design $(X_t)$. We assume the relation $y_t=X_t^\top\theta_t^{(m)}  +\varepsilon_t^{(m)}$, $t\ge 1$, associated with state equations $\theta_{t}^{(m)} = \theta_{t-1}^{(m)}=\cdots=\theta_{0}^{(m)}$ ($Q^{(m)}=0$ for $1\le m\le M$). Then the Kalman recursions are called static. Under {\bf (H)} the mean-variance identity Proposition \ref{prop:mean} implies $\sigma^{2(m)}=\sigma^2$. 
\item \underline{The dynamical setting:} This setting relies on the random walk state equations  
\begin{equation}\label{eq:rw}
\theta_{t}^{(m)} =  \theta_{t-1}^{(m)} +z_t^{(m)}\,,\qquad t\ge 1\,,
\end{equation}
from initial null state $\theta_0^{(m)}=0$ for all $1\le m\le M$.  Then the mean identity of Proposition \ref{prop:mean} is automatically satisfied as $\E[y_t]=0$ for all $t\ge 1$. The variance identity requires that $\E[X_t^\top Q^{(m)}X_t]=\E[X_t^\top Q^{(m')}X_t]$ and $\sigma^{2(m)}=\sigma^{2(m')}$ for any $1\le m,m'\le M$ and $t\ge 1$. The Kalman recursion can be used for tracking the signal $(y_t)$ on different explanatory variables $X_t^{(m)}$ stacked in $X_t$. 

\item \underline{The expert setting:} We consider the case where we have $M$ deterministic experts without any information about their generation process. This situation is very common in real-life applications as the forecast can come from different sources (physical models, different data sources, different machine learning models). 

For each $1\le m\le M$ expert we stack its prediction $f_{m,t}\in \R$ in $X_t^{(m)}$ together with the intercept and the past error $e_{m,t-1}=(y_{t-1}-f_{m,t-1})$, i.e.,
$$
X_t^{(m)}=(1,f_{m,t},e_{m,t-1}), \qquad t\ge 1\,.
$$

and each state-space model is defined by the state equation:

$$
\theta_{t}^{(m)} =  K^{(m)} \theta_{t-1}^{(m)} +z_t^{(m)} \,,\qquad t\ge 1.
$$

%
%We stack their predictions $f_{m,t}\in \R$ in $X_t$ together with the intercept and the past error $e_{m,t-1}=(y_{t-1}-f_{m,t-1})$, i.e.,
%$$
%X_t=(1,f_{1,t},\ldots,f_{M,t},e_{1,t-1},\ldots,e_{M,t-1})^\top\in \R^{3M}\,, \qquad t\ge 1\,.
%$$
%Each state model is depending on 6 parameters of the form
%$$
%\theta_t^{(m)}=K^{(m)}\theta_{t-1}^{(m)}+\eta_t^{(m)}\,,\qquad t\ge 1\,,
%$$
%with 
%\begin{align*}
%K^{(m)}&=\mbox{Diag}\Big(1,\underbrace{0,\ldots,0}_{m-1},1,0,\ldots,0,\rho,\underbrace{0,\ldots,0}_{M-m-1}\Big), \\
%\theta_t^{(m)}&=\Big(\theta_{t,1}^{(m)},\underbrace{0,\ldots,0}_{m-1} ,\theta_{t,2}^{(m)},0,\ldots,0,\theta_{t,3}^{(m)},\underbrace{0,\ldots,0}_{M-m-1}\Big)\,,\\
%Q^{(m)}&=\mbox{Diag}\Big(\sigma^{(m)}_1,0,\ldots,0,\sigma^{(m)}_2,\underbrace{0,\ldots,0}_{M-m-1}\Big)\\
%\theta_0^{(m)}&=\Big(\underbrace{0,\ldots,0}_{m} ,1,0,\ldots,0\Big)\,.
%\end{align*}

\end{enumerate}

\section{Kalman recursions Aggregated Online (KAO) algorithm}
Consider the state-space models (coinciding with Equation~(\ref{eq:stateq}) for the $ m$th state equation, $1\le 1\le M$)
\begin{equation}%\label{eq:ss2}
\left\{
\begin{array}{lcl}
y_t&=&X_t^\top\theta_t^{(m)}+\varepsilon_t^{(m)} \\
\theta_{t}^{(m)}&=&K^{(m)}\theta_{t-1}^{(m)}+z_t^{(m)} 
\end{array}
\right.,\quad t\geq 1,\quad 1\le  m\le M.\nonumber
\end{equation}
Recall that under {\bf (H)} we have
$$
\E[y_t\mid m]:=\E[y_t\mid z_t^{(m)},\ldots,z_1^{(m)},\mathcal F_{t-1}]=X_t^\top \theta_t^{(m)}, \qquad 1\le m\le  M.
$$
We aggregate Kalman recursions  using a version of the exponentially weighted average forecaster defined as \begin{equation}\label{eq:agregat_forecaster}
\hat{y}_t=\sum_{m=1}^M\rho_{t}^{(m)}\hat{y}_t^{(m)}
\end{equation}
with $\rho_{t}^{(m)}\geq 0$ for $1\le m\le M$, $\sum_{m=1}^M\rho_{t}^{(m)}=1$ and $\hat{y}_t^{(m)}=X_t^\top \theta_t^{(m)}$ is the $m$th Kalman forecaster of $y_t$. 

\subsection{Convex properties}
The ability to find rapidly the solution of an optimization problem depends heavily on the convex properties of the objective function. In our case, the objective function is the conditional risk defined in \eqref{eq:loss_function} as 
$
L_t(\hat y)=\E_{t-1}[(\hat{y}-y_t)^2].
$ 
Due to the conditional expectation, it is a random convex function and its minimum $\E_{t-1}[y_t]$, called the best prediction, varies upon the time $t\ge 1$. Thus one cannot expect that our procedure converges in general and we rather study its regrets $R_t^S$ and $R_t^A$ defined in Equations \eqref{regret:S} and \eqref{regret:A} as the model selection and aggregation regret, respectively.
The objective function is 
$$
\sum_{s=1}^t L_s\Big(\sum_{m=1}^M\pi^{(m)}\hat y^{(m)}_s\Big)
$$
for any $(\pi^{(m)})_{1\le m\le M}$ in  the canonical basis or in the simplex, i.e. $\pi^{(m)}\ge 0$ such that $\sum_{m=1}^M\pi^{(m)}=1$. The optimal rates of convergence in the model selection and the aggregation problems depend on the convex properties of the objective function and the observation of an approximation of the gradients. As the objective functions are convex, we will extensively use the gradient trick which consists to bounding the regret with the linearized risks $\mathscr{L}_s^{(m)}=L'_s(\hat y_s)(\hat y_s^{(m)}-\hat y_s)$ as 
\begin{equation}\label{eq:gradtrick}
L_t(\hat y_t)-L_t\Big(\sum_{m=1}^M\pi^{(m)}\hat y^{(m)}_t\Big)
\le - \sum_{m=1}^M\pi^{(m)}\mathscr{L}_t^{(m)}\,.
\end{equation}
Fast rates of convergence could be obtained easily if the objective function was strongly convex. Despite we use the square loss, it is not the case since the Hessian matrix 
$$
2\sum_{s=1}^t (\hat y_s^{(m)})_{1\le m\le M}(\hat y_s^{(m)})_{1\le m\le M}^\top
$$ 
of the objective function is a sum of rank-one matrices which are very unlikely to converge in any non-stationary settings. This issue is bypassed in online convex optimization thanks to the notion of exp-concavity extensively studied by \cite{hazan2016introduction}.
\begin{definition}
A loss function $\ell$ is  $\eta$-exp-concave (with $\eta>0$) on some convex set $\mathcal Y$ if the function $F(y)=\exp(-\eta\ell( y))$ is concave for all $y\in\mathcal Y$.
\end{definition}

We need to find out for which values of $\eta$ the conditional risks \eqref{eq:loss_function} are exp-concave. Moreover, we can use the exp-concave property of the risk $L_t$ to refine the gradient trick.
\begin{theorem}\label{thm:exp_concave}
Assume that it exists  $D>0 $ such that 
$|\hat{y}_t^{(m)}-\mu_t| \le D$ a.s. $1\le m\le M$,  $t\ge 1$ 
with $\mu_t=\E_{t-1}[y_t] $. Then  the conditional risk $L_t$ is a.s. $(2D^2)^{-1}$-exp-concave for any $y = \sum_{m=1}^M\pi^{(m)}\hat y^{(m)}_t$, $t\ge1$, with $(\pi^{(m)})_{1\le m\le M}$ in the simplex. Moreover if  the linearized risk satisfies $|\mathscr{L}_t^{(m)}| \le G^{(m)}$ for any $t\ge1$, $1\le m\le M$,  then we have
$$
L_t(\hat y_t)-L_t(\hat y_t^{(m)})\le -  \mathscr{L}_t^{(m)} - \eta^{(m)}{\mathscr{L}_t^{(m)}}^2\,,
$$
with $\eta^{(m)}=\frac{1}{8( 2G^{(m)}\vee D^2)}$, $1\le m\le M$. 
\end{theorem}
The refined linearized risk $ \mathscr{L}_t^{(m)} + \eta^{(m)}{\mathscr{L}_t^{(m)}}^2$ is called the surrogate risk. It is itself exp-concave due to the quadratic term whereas the linearized risk cannot be exp-concave.
\begin{proof}[Proof] Let $0<\eta\le\frac{1}{(2D^2)}$.
Consider the function $\varphi_\eta(y)=e^{-\eta L_t(y)}$ for  $y=\sum_{m=1}^M\pi^{(m)}\hat y^{(m)}_t$ and $(\pi^{(m)})_{1\le m\le M}$ in the simplex. The function $\varphi$ is a.s. at least twice differentiable and we have
$$
\varphi_\eta''(y)=-2\eta\varphi_\eta(y)\left(1-2\eta(y-\mu_t)^2\right).
$$
using the derivation under the integral sign.
For $\theta\in\Theta$, we get the concavity since $\varphi_\eta''(y)\le 0$ as 
$$2\eta(y-\mu_t)^2 \le 2\eta(\sum_{m=1}^M\pi^{(m)} (\hat y^{(m)}_t-\mu_t))^2\le 2\eta D^2= 1$$ and the first assertion follows.\\

We proceed as in the proof of Lemma 4.2 of \cite{hazan2016introduction} considering $\gamma^{(m)}=\frac{1}{2( 2G^{(m)}\vee D^2)}\le \eta$. One deduces from the concavity property of $\varphi_{\gamma^{(m)}}$ that $\varphi_{\gamma^{(m)}}(y)-\varphi_{\gamma^{(m)}}(z)\le \varphi_{\gamma^{(m)}}'(z)(y-z)$ which, taking $y=\hat y_t^{(m)}$ and $z=\hat y_t$, provides
\begin{multline*}
\exp(-\gamma^{(m)}L_t(\hat y_t^{(m)}))-\exp(-\gamma^{(m)}L_t(\hat y_t))\\
\le -\gamma^{(m)}L'_t(\hat y_t)\exp(-\gamma^{(m)}L_t(\hat y_t))(\hat y_t^{(m)}-\hat y_t).
\end{multline*}
One deduces   that 
$$
\gamma^{(m)}(L_t(\hat y_t)-L_t(\hat y_t^{(m)}))\le \log(1-\gamma^{(m)}L'_t(\hat y_t)(\hat y_t^{(m)}-\hat y_t)).
$$
Using the relation $\log(1-z)\le -z-\frac14 z^2$ that holds for any $|z|\le 1/4$ applied on $|\gamma^{(m)}L'_t(\hat y_t)(\hat y_t^{(m)}-\hat y_t)|\le 1/4$ one obtains 
$$
\gamma^{(m)}(L_t(\hat y_t)-L_t(\hat y_t^{(m)}))\le \gamma^{(m)}L'_t(\hat y_t)(\hat y_t-\hat y_t^{(m)})-\frac14(\gamma^{(m)}L'_t(\hat y_t)(\hat y_t-\hat y_t^{(m)}))^2
$$
and the second assertion follows.
\end{proof}

\subsection{KAO for model selection}
In this Section we assume the exp-concavity of the conditional risks and we adapt the classical analysis of the Exponentially Weighted Average (EWA) algorithm of \cite{cesa2006prediction} to our setting. The aggregation procedure, called KAO, is described in Algorithm \ref{alg:EWA}.

\begin{algorithm}[!t]
\caption{KAO for model selection}
    \label{alg:EWA}
\begin{flushleft}
{\bfseries Parameters:} The variances $\sigma^{2(m)}$, $1\le m\le M$ and the learning rate $\eta$.\\
{\bfseries Initialization:} The initial weights $\rho_1^{(m)}=\rho_0^{(m)}$, $1\le m\le M$.\\
 For each iteration $t=1,\dots,T$:\\
    {\bfseries Inputs:} The Kalman predictions $\hat y_{t+1}^{(m)}$ and the matrices  $P_{t}^{(m)}$, $1\le m\le M$. \\
    {\bfseries Recursion:}     Do:
\end{flushleft}
\begin{eqnarray*}
\rho_{t+1}^{(m)}&=&\frac{\exp\left(-\eta \left(X_{t}^\top P_{t}^{(m)}X_{t}+\sigma^{2(m)}\right)\right)\rho_{t}^{(m)}}{\sum_{m=1}^M \exp\left(-\eta\left( X_{t}^\top P_{t}^{(m)}X_{t}+\sigma^{2(m)}\right)\right)\rho_{t}^{(m)}}\\
\hat y_{t+1} &=& \sum_{m=1}^M\rho_{t+1}^{(m)}\hat y_{t+1}^{(m)}\,.
\end{eqnarray*}
\end{algorithm}

KAO achieves the optimal rate for model selection.
\begin{theorem}
Under assumption {\bf (H)} and if it exists  $D>0 $ such that 
$|\hat{y}_t^{(m)}-\mu_t| \le D$ a.s. $1\le m\le M$,  $t\ge 1$   then KAO for model selection with $\eta=\frac{1}{(2D^2)}$
achieves the regret bound
$$
R_t^S(m)\le - 2D^2  \log(\rho_0^{(m)})\qquad 1\le t\le T\,, 1\le m\le M\,.
$$
\end{theorem}
We note that the classical EWA algorithm satisfies a similar regret bound under the stronger assumption
$$
|\hat{y}_t^{(m)}-y_t| \le D,\qquad 1\le m\le M,\, t\ge 1,\, a.s.
$$
which never holds in our Gaussian setting. One usual way to bypass this well-known restriction of EWA is to use a doubling trick which deteriorates the regret bound, see \cite{cesa2006prediction} for more details.
\begin{proof}
The proof is standard and follows the line of the proof of the EWA regret in \cite{cesa2006prediction}. The crucial step consists in identifying the conditional risk of any Kalman prediction $\hat y^{(m)}_t$ under {\bf (H)}. We have the following Lemma
\begin{lem}
Under {\bf (H)} we have the identity $\hat y_t^{(m)}=\E_{t-1}[\E[y_t\mid m]]$.
\end{lem}
\begin{proof}
The Kalman recursion produces the best linear prediction which is equal to the conditional expectation in the gaussian case. Then we have $\hat y_t^{(m)}=\E_{t-1}[X_t \theta_t^{(m)}]=\E_{t-1}[\E[y_t|m]]$ by definition.
\end{proof}
We have explicitely
\begin{eqnarray*}
L_s(\hat y^{(m)}_s)&=&\E_{s-1}[(y_s-\E_{s-1}[\E[y_s\mid m]])^2] \\
&=&\E_{s-1}[(y_s-\E[y_s\mid m])^2] +\E_{s-1}[(\E[y_s\mid m]-\E_{s-1}[\E[y_s\mid m]])^2] \\
&&+2\E_{s-1}[(y_s-\E[y_s\mid m])(\E[y_s\mid m]-\E_{s-1}[\E[y_s\mid m]])] \\
&=&\E_{s-1}[(y_s-X_t \theta_t^{(m)})^2] +\E_{s-1}[(X_t( \theta_t^{(m)}-\hat \theta_t^{(m)})^2] \\
&=&\sigma^{2(m)}+X_s^\top P_s^{(m)}X_s\,, 
\end{eqnarray*}
since the third term of the sum is zero and since $\E_{s-1}[(X_t( \theta_t^{(m)}-\hat \theta_t^{(m)})^2]=X_s^\top P_s^{(m)}X_s$ thanks to the Kalman recursion properties in the gaussian case. We also have, using the exp-concavity of $L_t$ and Jensen inequality,
\begin{eqnarray*}
e^{-\eta L_t(\hat{y}_t)}&=&e^{-\eta L_t(\sum_{m=1}^M\rho_{t}^{(m)}\hat{y}_t^{(m)})}\\
&\geq& \sum_{m=1}^M\rho_{t}^{(m)}e^{-\eta L_t (\hat{y}_t^{(m)})}, \text{ }\\
&\ge&\frac{\sum_{m=1}^M \rho_{0}^{(m)}e^{-\eta \sum_{s=1}^{t-1}L_s(\hat y_s)} e^{-\eta L_t (\hat{y}_t^{(m)})}}{\sum_{m=1}^M \rho_{0}^{(m)}e^{-\eta\sum_{s=1}^{t-1}L_s(\hat y_s)}}\\
&\ge&\frac{\sum_{m=1}^M \rho_{0}^{(m)}e^{-\eta R_{t-1}^{S}(m)}e^{-\eta L_t (\hat{y}_t^{(m)})}}{\sum_{m=1}^M \rho_{0}^{(m)}e^{-\eta R_{t-1}^{S}(m)}}\,,  
\end{eqnarray*}
multiplying by $e^{\eta\sum_{s=1}^{t-1}L_s(\hat y_s^{(m)})}$ above and below the fraction. We get the recursive relation
$$
1=\sum_{m=1}^M \rho_{0}^{(m)}\ge \sum_{m=1}^M \rho_{0}^{(m)} e^{-\eta R_{t-1}^{S}(m)}\ge \sum_{m=1}^M \rho_{0}^{(m)}e^{-\eta R_{t}^{S}(m)}\,
$$
and the desired result follows.
\end{proof}

\subsection{KAO for aggregation}

In the case where the best expert is not worthy of confidence, it is generally much more interesting to compete with the best convex combination of the experts at hand. In this context, the aim is to provide a bound on the regret for aggregation $R_t^A(\pi)$ where $\pi:=(\pi^{(m)})_{1\le m\le M}$ belongs to the simplex. As the conditional risk  $L_s$  is a convex function that is differentiable, one applies the gradient trick and we consider an explicit biased version of the linearized risk 
\begin{multline}\label{eq:centeredpseudoloss}
\mathscr{L}_t^{(m)}=X_{t}^\top P_{t}^{(m)}X_{t}+\sigma^{2(m)}- (\hat{y}_{t}-\hat{y}_{t}^{(m)})^2\\-\sum_{m'=1}^M\rho_t^{(m')} \left(X_{t}^\top P_{t}^{(m')}X_{t}+\sigma^{2(m')}- (\hat{y}_{t}-\hat{y}_{t}^{(m')})^2\right).
\end{multline}
By convention $\mathscr{L}_0^{(m)}=0$.
In our setting, the adaptation of the gradient-based EWA of \cite{cesa2006prediction} yields Algorithm \ref{alg:gt}.
\begin{algorithm}[!t]
\caption{KAO for aggregation}
    \label{alg:gt}
\begin{flushleft}
{\bfseries Parameters:} The variances $\sigma^{2(m)}$, $1\le m\le M$ and the learning rate $\eta$.\\
{\bfseries Initialization:} The initial weights $\rho_0^{(m)}$, $1\le m\le M$.\\
 For each iteration $t=0,\dots,T$:\\
    {\bfseries Inputs:} The Kalman predictions $\hat y_{t+1}^{(m)}$ and the matrices  $P_{t}^{(m)}$, $1\le m\le M$. \\
    {\bfseries Recursion:}     Do:
\end{flushleft}
\begin{eqnarray*}
\mathscr{L} _t^{(m)} &=& \eqref{eq:centeredpseudoloss}\\
\rho_{t+1}^{(m)}&=&\frac{\exp\left(-\eta\mathscr{L}_t^{(m)}\right)\rho_{t}^{(m)}}{\sum_{m'=1}^M \exp\left(-\eta \mathscr{L}_t^{(m)}\right)\rho_{t}^{(m')}}\\
\hat y_{t+1} &=& \sum_{m=1}^M\rho_{t+1}^{(m)}\hat y_{t+1}^{(m)}\,.
\end{eqnarray*}
\end{algorithm}
The following theorem derives an upper bound for the regret $R_t^A$.
\begin{theorem}
Under Assumption {\bf (H)}, suppose it exists $G>0$ such that 
$|\mathscr{L}_t^{(m)}|\le G$  a.s. for $1\le t\le T$, $1\le m\le M$.
Then KAO for aggregation starting with $\rho_0^{(m)}=1/M$ and $\eta=\frac{1}{G}\sqrt{\frac{2\log M}{t}}$ satisfies the regret bound
\begin{equation}\label{eq:regretbound2}
%R_t^A\le G \sqrt{\dfrac{2\log M}{t}}\,, \qquad 1\le t\le T.
R_t^A\le G \sqrt{2t\log M}\,, \qquad 1\le t\le T.
\end{equation}
\end{theorem}
The regret bound matches the optimal bound for $M\ge \sqrt t$. Note that the boundedness assumption on $\mathscr{L}_t^{(m)} $ involves only the predictions and does not require to bound $(y_t)$. 
\begin{proof}
Since $L_s$ is convex and differentiable, we apply the gradient trick
$$
R_t^A(\pi)\le - \sum_{s=1}^t  \sum_{m=1}^M \pi^{(m)}\mathscr{L}_s^{(m)}.
$$
Moreover, the expression of $L_s'(\hat{y}_s)\hat{y}_s^{(m)}$ can be developed as
\begin{eqnarray*}
L_s'(\hat{y}_s)\hat{y}_s^{(m)}&=&2\E_{s-1}[(\hat{y}_s-y_s)\hat{y}_s^{(m)}]\\
&=&2\hat{y}_s\hat{y}_s^{(m)} + \E_{s-1}[(y_s-\hat{y}_s^{(m)})^2] - \hat{y}_s^{(m)^2} - \E_{s-1}[y_s^2]\\
&=&\E_{s-1}[(y_s-\hat{y}_s^{(m)})^2] - (\hat{y}_s-\hat{y}_s^{(m)})^2 + \hat{y}_s^2 - \E_{s-1}[y_s^2]\\
&=&X_s^{ \top} P_s^{(m)}X_s +\sigma^{2(m)} - (\hat{y}_s-\hat{y}_s^{(m)})^2 + \hat{y}_s^2 - \E_{s-1}[y_s^2].
\end{eqnarray*} 
Since the two last summands of $L_s'(\hat{y}_s)\hat{y}_s^{(m)}$ do not depend on $m$ we obtain the identity $\mathscr{L}_t^{(m)}=L_s'(\hat{y}_s)(\hat{y}_s^{(m)}-\hat y_s)=\eqref{eq:centeredpseudoloss}$. As it exists $G>0$ satisfying $|\mathscr{L}_t^{(m)}|\le G$, by using the Hoeffding lemma (i.e. $\log\E[e^{\alpha X}]\le \frac{\alpha^2}{2}G^2$, for any centered random variable $|X|\le G$, with $\alpha\in\mathbb{R}$), and the identity 
$$
\rho_s^{(m)}=\dfrac{\exp(-\eta \sum_{r=1}^{s-1}\mathscr{L}_t^{(m)})\rho_0^{(m)}}{\sum_{m'= 1}^M \exp(-\eta \sum_{r=1}^{s-1}\mathscr{L}_t^{(m')})\rho_0^{(m')}}
$$ we get
$$
\log\left(\sum_{m=1}^M\frac{\exp\left(-\eta\sum_{r=1}^{s}\mathscr{L}_r^{(m)}\right)\rho_0^{(m)}}{\sum_{m'=1}^M \exp\left(-\eta\sum_{r=1}^{s-1}\mathscr{L}_r^{(m')}\right)\rho_0^{(m')}}\right) \le  \frac{\eta^2}{2}G^2\,.
$$
Then, by summing over $s$, a telescoping sum appears and leads to
$$
\dfrac1\eta \log\left( \sum_{m=1}^M\exp\left(-\eta\sum_{s=1}^t \mathscr{L}_s^{(m)} \right)\rho_0^{(m)}\right)
 \le  \eta t\frac{G^2}{2}.
$$
Moreover, 
$$
%\max_{1\le m\le M}\exp\left(-\eta\sum_{s=1}^t  \mathscr{L}_s^{(m)}\right)\rho_0^{(m)}\le \sum_{m=1}^M\exp\left(-\eta\sum_{s=1}^t \mathscr{L}_s^{(m)}\right)\rho_0^{(m)},
\exp\left(-\eta\sum_{s=1}^t  \mathscr{L}_s^{(m)}\right)\rho_0^{(m)}\le \sum_{m=1}^M\exp\left(-\eta\sum_{s=1}^t \mathscr{L}_s^{(m)}\right)\rho_0^{(m)},
$$
which leads to
$$
\sum_{m=1}^M\pi^{(m)}\left(\sum_{s=1}^t- \mathscr{L}_s^{(m)}+\frac{\log \rho_0^{(m)}}{\eta}\right)\le \frac{1}{\eta}\log\left(\sum_{m=1}^M\exp\left(-\eta\sum_{s=1}^t \mathscr{L}_s^{(m)} \right)\rho_0^{(m)}\right).
$$
Combining those bounds, we obtain
$$
R_t^A(\pi)\le  \sum_{s=1}^t  \sum_{m=1}^M - \pi^{(m)}\mathscr{L}_s^{(m)}\le  \sum_{m=1}^M\pi^{(m)}\left(\frac{-\log \rho_0^{(m)}}{\eta}+ \eta t\frac{G^2}{2}\right).
$$
We get the desired result noticing that $\rho_0^{(m)}= 1/M$  
and the optimal choice of $\eta=\frac{1}{G}\sqrt{\frac{2\log M}{t}}$.
\end{proof}
We notice that a unique learning rate yields a uniform regret bound, independent of $\pi$. We also notice that we can use the gradient trick despite we only observed a biased version of the linearized risk thanks to the exponential form of the weights that are not sensitive to the  bias.

\section{Online tuning of the learning rates}
The theoretical guarantees on the regret for the model selection and the aggregation problems do not hold for the same algorithm. The gradient trick is a crucial step in the proof of the regret for the aggregation problem. However, the fast rate for the model selection does not hold for the gradient-based EWA since the linearized risk cannot be exp-concave. In order to bypass this issue, we adapt the approach of \cite{wintenberger2017optimal} to our setting. The first step is to use a surrogate loss of the form $\mathscr{L}_t^{(m)}\left(1+\eta \mathscr{L}_t^{(m)}\right)$ where the quadratic part yields exp-concavity. The second step is to use a multiple learning rates version of KAO as described in Algorithm \ref{alg:ml} in Section \ref{sec:ml} where we show that multiple learning rates can be easily tuned online.

\begin{algorithm}[!t]
\caption{KAO with multiple learning rates}
    \label{alg:ml}
\begin{flushleft}
{\bfseries Parameters:} The variances $\sigma^{2(m)}$, the weights $\tilde \rho_{0}^{(m)}$ and the learning rates $\eta^{(m)}$, $1\le m\le M$.\\
{\bfseries Initialization:} The initial weights $\rho_{0}^{(m)}=\eta^{(m)}\tilde \rho_{0}^{(m)}/(\sum_{m'=1}^M\eta^{(m')}\tilde \rho_{0}^{(m')})$, $1\le m\le M$.\\
 For each iteration $t=1,\dots,T$:\\
    {\bfseries Inputs:} The Kalman predictions $\hat y_{t+1}^{(m)}$ and the matrices  $P_{t}^{(m)}$, $1\le m\le M$. \\
    {\bfseries Recursion:}     Do:
\end{flushleft}
\begin{eqnarray*}
\mathscr{L}_t^{(m)} &=& \eqref{eq:centeredpseudoloss}\\
\rho_{t+1}^{(m)}&=&\frac{\exp\left(-\eta^{(m)} \mathscr{L}_t^{(m)}(1+\eta^{(m)}\mathscr{L}_t^{(m)}) \right)\rho_{t}^{(m)}}{\sum_{m'=1}^M \exp\left(-\eta^{(m)} \mathscr{L}_t^{(m')}(1+\eta^{(m')}\mathscr{L}_t^{(m')}) \right)\rho_{t}^{(m')}}\\
\hat y_{t+1} &=& \sum_{m=1}^M\rho_{t+1}^{(m)}\hat y_{t+1}^{(m)}\,.
\end{eqnarray*}
\end{algorithm}

\subsection{Multiple learning rates for KAO}\label{sec:ml}
In the context of expert aggregation, it is well known that using multiple learning rates help to increase the prediction accuracy, see \cite{gaillard2014second} and \cite{wintenberger2017optimal}. Here we aim to provide a multiple learning rates version for KAO in a similar way than the multiple learning rate version of the BOA procedure (see \cite{wintenberger2017optimal}).  The following theorem provides regret bounds both for model selection and aggregation on the same algorithm.

\begin{theorem}
Under assumption {\bf (H)} suppose it exists $G>0$ such that 
$|\mathscr{L}_t^{(m)}|\le G$  a.s. for $1\le t\le T$, $1\le m\le M$. Then the aggregation regret of KAO with multiple learning rates $\eta^{(m)}=\frac{1}{G}\left(\sqrt{-\frac{\log \tilde\rho^{(m)}_0}{t}}\wedge \frac12\right)$ is bounded as
\begin{equation}\label{eq:regretbound2}
R_t^A(\pi)\leq 2G\sum_{m=1}^M \pi^{(m)}\left(\sqrt{ - \log \tilde \rho_0^{(m)}t}- \log \tilde \rho_0^{(m)}\right)\,,\qquad 1\le t\le T\,.
\end{equation}
If moreover there exists  $D>0 $ such that 
$|\hat{y}_t^{(m)}-\mu_t| \le D$ a.s. $1\le m\le M$,  $t\ge 1$  then the model selection regret of KAO  with multiple learning rates $\eta^{(m)} = \frac{1}{8( 2G\vee D^2)} $ for all $1\leq m\leq M$  is bounded as
\begin{equation} 
R_t^S(m) \leq   - 8( 2G\vee D^2)\log \tilde \rho^{(m)} \,.
\end{equation}
\end{theorem}
\begin{proof}
We start by applying the gradient trick as in \eqref{eq:gradtrick} inferring that
$$
R_t^A(\pi)\le  -\sum_{s=1}^t  \sum_{m=1}^M\pi^{(m)}\mathscr{L}_s^{(m)}.
$$
Moreover, as $x-x^2$ is $1$-exp-concave for $x>1/2$, Jensen's inequality implies that
\begin{equation}\label{eq:expoineq}
\E\left[\exp\left(X-X^2\right)\right]\le \exp\left(\E[X]-\E[X]^2\right)=1
\end{equation}
for any centered random variable $X$ such that $X\ge -1/2$ a.s. We notice that $\eta^{(m)}=\frac{1}{G}\sqrt{-\frac{\log \tilde\rho^{m}_0}{t}}\wedge \frac12$ satisfies the relation
$$
\eta^{(m)}\mathscr{L}_t^{(m)}\le \frac{1}{2}\,.
$$
Denoting
$$
\tilde{\rho}_t^{(m)}=\frac{\exp\left[-\sum_{s=1}^{t-1}\eta^{(m)}\mathscr{L}_s^{(m)} \left(1+\eta^{(m)}\mathscr{L}_s^{(m)}\right)\right]\tilde{\rho}_0^{(m)}}{\sum_{m'=1}^M \exp\left[-\sum_{s=1}^{t-1}\eta^{(m')}\mathscr{L}_s^{(m')} \left(1+\eta^{(m')}\mathscr{L}_s^{(m')}\right)\right]\tilde{\rho}_0^{(m')}},
$$
we have the identity
$$
\tilde{\rho}_t^{(m)}=\frac{\rho_t^{(m)}}{\eta^{(m)}}\times\frac{1}{\sum_{m'=1}^M \rho_t^{(m')}/\eta^{(m')}},
$$
which leads to
$$
\sum_{m=1}^M\tilde{\rho}_t^{(m)}\eta^{(m)}\mathscr{L}_t^{(m)}=0.
$$
Thus the random variable $\left(\eta^{(m)}\mathscr{L}_t^{(m)}\right)_{1\le m\le M}$ is therefore centered for the distribution $\left(\tilde{\rho}_t^{(m)}\right)_{1\le m\le M}$, and we have
\begin{equation}\label{eq:hypothesis1}
\sum_{m=1}^M\tilde{\rho}_t^{(m)}\exp\left[ -\eta^{(m)}\mathscr{L}_t^{(m)} \left(1+\eta^{(m)}\mathscr{L}_t^{(m)}\right)\right]\leq 1,
\end{equation}
since $-\eta^{(m)}\mathscr{L}_t^{(m)}>-1/2$ a.s. for all $1\le m\le M$ and $1\le t\le T$, using~(\ref{eq:expoineq}).
By putting the expression of $\tilde{\rho}_t^{(m)}$ into Equation~(\ref{eq:hypothesis1}), we have
\begin{multline}
\sum_{m=1}^M\tilde{\rho}_0^{(m)}\exp\left[ -\sum_{s=1}^t \eta^{(m)}\mathscr{L}_s^{(m)} \left(1+\eta^{(m)}\mathscr{L}_s^{(m)}\right)\right] \le \\ \sum_{m=1}^M\tilde{\rho}_0^{(m)}\exp\left[ -\sum_{s=1}^{t-1} \eta^{(m)}\mathscr{L}_s^{(m)} \left(1+\eta^{(m)}\mathscr{L}_s^{(m)}\right)\right],\nonumber
\end{multline}
which implies for $1\le t\le T$ that
\begin{equation}\nonumber
\sum_{m=1}^M\tilde{\rho}_0^{(m)}\exp\left[ -\sum_{s=1}^t \eta^{(m)}\mathscr{L}_s^{(m)} \left(1+\eta^{(m)}\mathscr{L}_s^{(m)}\right)\right] \le 1,
\end{equation}
since  by convention that $\mathscr{L}(\hat{y}_0^{(m)})=0$. Thus, for $1\le m\le M$, 
\begin{equation}\label{eq:exploss}
-\sum_{s=1}^t \eta^{(m)}\mathscr{L}_s^{(m)} \left(1+\eta^{(m)}\mathscr{L}_s^{(m)}\right) \le -\log \tilde{\rho}_0^{(m)},
\end{equation}
by applying the logarithm function using the previous inequality. We have
$$
-\sum_{m=1}^M\tilde \pi^{(m)}\sum_{s=1}^t \eta^{(m)}\mathscr{L}_s^{(m)}   \le \sum_{m=1}^M\tilde \pi^{(m)}\left( {\eta^{(m)}}^2\sum_{s=1}^t{\mathscr{L}_s^{(m)}}^2  -\log \tilde{\rho}_0^{(m)}\right)\,
$$
for 
$$
\tilde \pi^{(m)}=\dfrac{\pi^{(m)}/\eta^{(m)}}{\sum_{m'=1}^M\pi^{(m')}/\eta^{(m')}}\,,\qquad 1\le m\le M\,.
$$
Multiplying with $\sum_{m'=1}^M\pi^{(m')}/\eta^{(m')}$ we obtain
\begin{align*}
-\sum_{m=1}^M \pi^{(m)}\sum_{s=1}^t \mathscr{L}_s^{(m)}   &\le \sum_{m=1}^M  \pi^{(m)}\left( \eta^{(m)}\sum_{s=1}^t{\mathscr{L}_s^{(m)}}^2 - \dfrac{ \log \tilde{\rho}_0^{(m)}}{\eta^{(m)}}\right)\\
&\le \sum_{m=1}^M  \pi^{(m)}\left( \eta^{(m)} G^2 t - \dfrac{ \log \tilde{\rho}_0^{(m)}}{\eta^{(m)}}\right)
\end{align*}
and the desired result on $R_t^A(\pi)$ follows from the specific choice of $\eta^{(m)}$.

The regret bound on $R_t^S(m)$ follows by an application of Theorem \ref{thm:exp_concave} on the last bound specified for $\pi$ in the canonical basis
$$
\sum_{s=1}^t- \mathscr{L}_s^{(m)}- \eta^{(m)}{\mathscr{L}_s^{(m)}}^2   \le  \pi^{(m)}\left(  - \dfrac{ \log \tilde{\rho}_0^{(m)}}{\eta^{(m)}}\right)
$$
for all $1\le m\le M$ as $\eta^{(m)}= \frac{1}{8( 2G\vee D^2)} $.
\end{proof}
However the regrets bound do not apply on KAO with the same learning rates. The slow rate aggregation regret bound  holds for a $O(1/\sqrt t)$ learning rate whereas the fast rate  model selection regret bound  holds for a constant learning rate.

\subsection{Adaptive multiple learning rates}\label{sec:mlada}
Multiple learning rates are easily adaptable as in Algorithm  \ref{alg:mlada}. Moreover, a single algorithm with unique adaptive learning rates achieves optimal regret bounds for both model selection and aggregation problems as for the BOA algorithm developed by \cite{wintenberger2017optimal} and refined by \cite{gaillard2018efficient}.
\begin{algorithm}[!t]
\caption{KAO with adaptive multiple learning rates}
    \label{alg:mlada}
\begin{flushleft}
{\bfseries Parameters:} The variances $\sigma^{2(m)}$, $1\le m\le M$.\\
{\bfseries Initialization:} Any initial weights $\tilde \rho_{0}^{(m)}>0$ such that $\sum_{m=1}^M\tilde \rho_0^{(m)}=1$, $1\le m\le M$.\\
 For each iteration $t=1,\dots,T$:\\
    {\bfseries Inputs:} The Kalman predictions $\hat y_{t+1}^{(m)}$ and the matrices  $P_{t}^{(m)}$, $1\le m\le M$. \\
    {\bfseries Recursion:}     Do:
\end{flushleft}
\begin{eqnarray*}
\mathscr{L}_t^{(m)}&=&\eqref{eq:centeredpseudoloss}\\
 \eta_t^{(m)}&=&\sqrt{\frac{-\log \tilde \rho_0^{(m)}}{1+\sum_{s=1}^t {\mathscr{L}_s^{(m)}}^2}}\\
\rho_{t+1}^{(m)}&=&\frac{\eta_{t}^{(m)}\exp\left[-\eta_{t}^{(m)}\sum_{s=1}^{t}\mathscr{L}_s^{(m)} \left(1+\eta_{s-1}^{(m)}\mathscr{L}_s^{(m)}\right)\right]\tilde{\rho}_0^{(m)}}{\sum_{m'=1}^M \eta_{t}^{(m')}\exp\left[-\eta_{t}^{(m')}\sum_{s=1}^{t}\mathscr{L}_s^{(m')} \left(1+\eta_{s-1}^{(m')}\mathscr{L}_s^{(m')}\right)\right]\tilde{\rho}_0^{(m')}}\\
\hat y_{t+1} &=& \sum_{m=1}^M\rho_{t+1}^{(m)}\hat y_{t+1}^{(m)}\,.
\end{eqnarray*}
\end{algorithm}

\begin{theorem}
Under assumption {\bf (H)} suppose there exist $G^{(m)}>0$ and $D>0$ such that 
$|\mathscr{L}_t^{(m)}|\le G^{(m)}$ and $|\hat y_t^{(m)}-\mu_t|\le D$ a.s.  for $1\le t\le T$, $1\le m\le M$. Then the regret of KAO with adaptive multiple learning rates such that $ \eta_{t-1}^{(m)} \mathscr{L}_t^{(m)}<1/2$ for any $1\le t\le T$ and $1\le m \le M$ is bounded as
\begin{align*}
R_t^A(\pi)\leq & \sum_{m=1}^M \pi^{(m)}\left(G^{(m)}(3+G^{(m)})\sqrt t+1\right)\left( \sqrt{-\log \tilde \rho_0^{(m)}}+r_t^{(m)}\right)\,,\\
R_t^S(m)\le&\, 8(2 G^{(m)}\vee D^2) (3+G^{(m)}) \left(\sqrt{-\log \tilde \rho_0^{(m)}}+r_t^{(m)}\right)^2\\
&\qquad\qquad\qquad\qquad\qquad\qquad\qquad\qquad+ \sqrt{-\log \tilde \rho_0^{(m)}}+r_t^{(m)} \,.
\end{align*}
where $r_t^{(m)}=\frac{\log\log\left( e^{1/4} +G^{(m)} \sqrt{t+1}\right)}{\sqrt{-\log \tilde \rho_0^{(m)}}}$.
\end{theorem}
\begin{remark}\label{rem:dt}
The leading constant is proportional to ${G^{(m)}}^2$. It is not optimal and can be reduced to ${G^{(m)}}$ by refining the adaptive learning rates as in \cite{cesa2007improved}.
\end{remark} 
\begin{proof}
By adapting the inequality \eqref{eq:hypothesis1} as $(\eta_{t-1}^{(m)}\mathscr{L}_ t^{(m)})_{1\le m\le M}$ is centered for $\left(\tilde{\rho}_t^{(m)}\right)_{1\le m\le M}$, where
$$
\tilde{\rho}_t^{(m)}=\frac{\exp\left[-\eta_{t-1}^{(m)}\sum_{s=1}^{t-1}\mathscr{L}_s^{(m)} \left(1+\eta_{s-1}^{(m)}\mathscr{L}_s^{(m)}\right)\right]\tilde{\rho}_0^{(m)}}{\sum_{m'=1}^M \exp\left[-\eta_{t-1}^{(m')}\sum_{s=1}^{t-1}\mathscr{L}_s^{(m')} \left(1+\eta_{s-1}^{(m')}\mathscr{L}_s^{(m')}\right)\right]\tilde{\rho}_0^{(m')}},
$$
for any $t\ge 2 $ we have
\begin{multline}\label{eq:ineq1}
\sum_{m=1}^M \tilde{\rho}_0^{(m)}\exp\left[-\eta_{t-1}^{(m)}\sum_{s=1}^{t}\mathscr{L}_s^{(m)} \left(1+\eta_{s-1}^{(m)}\mathscr{L}_s^{(m)}\right)\right]\le\\
\sum_{m=1}^M \tilde{\rho}_0^{(m)}\exp\left[-\eta_{t-1}^{(m)}\sum_{s=1}^{t-1}\mathscr{L}_s^{(m)} \left(1+\eta_{s-1}^{(m)}\mathscr{L}_s^{(m)}\right)\right].
\end{multline}
Since $x\le x^\alpha+\alpha^{-1}(\alpha-1)$ for  $x\ge 0$ and $\alpha\ge1$, by setting  
$$
\alpha=\frac{\eta_{t-2}^{(m)}}{\eta_{t-1}^{(m)}} \, \text{ and }\, x=\exp\left[-\eta_{t-1}^{(m)}\sum_{s=1}^{t-1}\mathscr{L}_s^{(m)} \left(1+\eta_{s-1}^{(m)}\mathscr{L}_s^{(m)}\right)\right],
$$
we have for any $t\ge 2$ the relation
\begin{multline*}
\exp\left[-\eta_{t-1}^{(m)}\sum_{s=1}^{t-1}\mathscr{L}_s^{(m)} \left(1+\eta_{s-1}^{(m)}\mathscr{L}_s^{(m)}\right)\right]\le\\
\exp\left[-\eta_{t-2}^{(m)}\sum_{s=1}^{t-1}\mathscr{L}_s^{(m)} \left(1+\eta_{s-1}^{(m)}\mathscr{L}_s^{(m)}\right)\right] + \frac{\eta_{t-2}^{(m)}-\eta_{t-1}^{(m)}}{\eta_{t-2}^{(m)}},
\end{multline*}
which leads to
\begin{multline}\label{eq:ineq2}
\sum_{m=1}^M \tilde{\rho}_0^{(m)}\exp\left[-\eta_{t-1}^{(m)}\sum_{s=1}^{t}\mathscr{L}_s^{(m)} \left(1+\eta_{s-1}^{(m)}\mathscr{L}_s^{(m)}\right)\right]\le\\
\sum_{m=1}^M \tilde{\rho}_0^{(m)}\exp\left[-\eta_{t-2}^{(m)}\sum_{s=1}^{t-1}\mathscr{L}_s^{(m)} \left(1+\eta_{s-1}^{(m)}\mathscr{L}_s^{(m)}\right)\right]+\sum_{m=1}^M\tilde{\rho}_0^{(m)} \frac{\eta_{t-2}^{(m)}-\eta_{t-1}^{(m)}}{\eta_{t-2}^{(m)}}.
\end{multline}
Using a recursion argument on $t\ge 2$ on  Equation~(\ref{eq:ineq2})  yields
\begin{multline}\label{eq:ineq3}
\sum_{m=1}^M \tilde{\rho}_0^{(m)}\exp\left[-\eta_{t-1}^{(m)}\sum_{s=1}^{t}\mathscr{L}_s^{(m)} \left(1+\eta_{s-1}^{(m)}\mathscr{L}_s^{(m)}\right)\right]\\
\le \sum_{m=1}^M \tilde{\rho}_0^{(m)}\exp\left[-\eta_{0}^{(m)}\mathscr{L}_1^{(m)} \left(1+\eta_{0}^{(m)}\mathscr{L}_1^{(m)}\right)\right]+\sum_{s=1}^{t-1} \sum_{m=1}^M\tilde{\rho}_0^{(m)}\frac{\eta_{s-1}^{(m)}-\eta_{s}^{(m)}}{\eta_{s-1}^{(m)}}.
\end{multline}
Moreover, we have 
$$
\sum_{s=1}^{t-1}\frac{\eta_{s-1}^{(m)}-\eta_{s}^{(m)}}{\eta_{s-1}^{(m)}}\le\sum_{s=1}^{t-1}\int_{\eta_{s}^{(m)}}^{\eta_{s-1}^{(m)}}\frac{dx}{x}\le\int_{\eta_{t-1}^{(m)}}^{\eta_{0}^{(m)}}\frac{dx}{x}\le\log\left(\frac{\eta_{0}^{(m)}}{\eta_{t-1}^{(m)}}\right)
$$
the estimate of the ratio
$$
\frac{\eta_{0}^{(m)}}{\eta_{t-1}^{(m)}}=\sqrt{1+ \sum_{s=1}^{t-1}{\mathscr{L}_s^{(m)}}^2} \le G^{(m)} \sqrt{t+1},\, \text{ for } G^{(m)}\ge 1.
$$
and 
$$
-\eta_{0}^{(m)}\mathscr{L}_1^{(m)} \left(1+\eta_{0}^{(m)}\mathscr{L}_1^{(m)}\right)\le 1/4\,.
$$
Equation~(\ref{eq:ineq3}) implies
$$
-\eta_{t-1}^{(m)}\sum_{s=1}^{t}\mathscr{L}_s^{(m)} \left(1+\eta_{s-1}^{(m)}\mathscr{L}_s^{(m)}\right)\le -\log \tilde \rho_0^{(m)}
+r_t^{(m)},
$$
and we obtain similarly than above that for any $\pi$ we have 
$$
-\sum_{m=1}^M\pi^{(m)} \sum_{s=1}^{t}\mathscr{L}_s^{(m)} 
\le\sum_{m=1}^M\pi^{(m)}\left( \sum_{s=1}^{t} \eta_{s-1}^{(m)}{\mathscr{L}_s^{(m)}}^2+ \frac{-\log \tilde \rho_0^{(m)}+r_t^{(m)}}{\eta_{t-1}^{(m)}}\right)\,.
$$
In order to bound the second order term $ \sum_{s=1}^{t} \eta_{s-1}^{(m)}{\mathscr{L}_s^{(m)}}^2$ we denote $V_t=1+\sum_{s=1}^{t} {\mathscr{L}_s^{(m)}}^2$ so that
\begin{align*}
 \eta_{s-1}^{(m)}{\mathscr{L}_s^{(m)}}^2&=\sqrt{-\log \tilde \rho_0^{(m)}}\dfrac{V_s-V_{s-1}}{\sqrt{V_{s-1}}}\\
 &=\sqrt{-\log \tilde \rho_0^{(m)}}\dfrac{\sqrt V_s + \sqrt V_{s-1}}{\sqrt{V_{s-1}}}(\sqrt V_s - \sqrt V_{s-1})\\
 &=\sqrt{-\log \tilde \rho_0^{(m)}}\Big(\sqrt{V_s/V_{s-1}} +1\Big)(\sqrt V_s - \sqrt V_{s-1})\\
  &\le \sqrt{-\log \tilde \rho_0^{(m)}}\Big(\sqrt{1+{G^{(m)}}^2} +1\Big)(\sqrt V_s - \sqrt V_{s-1})\,.
\end{align*}
A telescoping sum argument yields
\begin{align*}
 \sum_{s=1}^{t} \eta_{s-1}^{(m)}{\mathscr{L}_s^{(m)}}^2&\le (2+G^{(m)}) \sqrt{-\log \tilde \rho_0^{(m)}}\left(\sqrt{1+\sum_{s=1}^{t} {\mathscr{L}_s^{(m)}}^2}-1\right)\\
&\le  (2+G^{(m)}) \sqrt{-\log \tilde \rho_0^{(m)}\sum_{s=1}^{t} {\mathscr{L}_s^{(m)}}^2}\,.
\end{align*}
Finally we get
\begin{align*}
-\sum_{m=1}^M\pi^{(m)} \sum_{s=1}^{t}\mathscr{L}_s^{(m)} 
\le&\sum_{m=1}^M\pi^{(m)}\left( \sqrt{\sum_{s=1}^{t} {\mathscr{L}_s^{(m)}}^2}\left( (3+G^{(m)}) \sqrt{-\log \tilde \rho_0^{(m)}}\right.\right.\\
&\left.\left.+r_t^{(m)}\right) + \sqrt{-\log \tilde \rho_0^{(m)}}+r_t^{(m)}\right)\,
\end{align*}
and the desired bound on $R_t^A$ follows.

In order to obtain the regret bound on $R_t^S$ we use the Young inequality $2\sqrt{ab} \le \gamma a+b/\gamma $ with $\gamma=4(2 G^{(m)}\vee D^2)$ so that 
\begin{align*}
-\sum_{m=1}^M\pi^{(m)} \sum_{s=1}^{t}\mathscr{L}_s^{(m)} 
\le&\sum_{m=1}^M\pi^{(m)}\left( \dfrac1{8(2 G^{(m)}\vee D^2)}\sum_{s=1}^{t} {\mathscr{L}_s^{(m)}}^2\right. \\
& +8(2 G^{(m)}\vee D^2) (3+G^{(m)}) \left(\sqrt{-\log \tilde \rho_0^{(m)}}+r_t^{(m)}\right)^2\\
&\left. + \sqrt{-\log \tilde \rho_0^{(m)}}+r_t^{(m)}\right)
\end{align*}
and the desired result follows from an application of Theorem \ref{thm:exp_concave}.
\end{proof}

\section{Discussion and examples}
KAO algorithms require the knowledge of the variances $\sigma^{2(m)}>0$.
 A natural estimator of this  quantity is the mean square residuals 
$$
\widehat\sigma^{2(m)}_t=\frac1t \sum_{s=1}^t \big(y_s-\hat y_s^{(m)}\big)^2\,.
$$
It can be tuned online but without any guarantee on the regret of the corresponding algorithm. In our applications, we prefer to estimate $\widehat\sigma^{2(m)}_t$ on a burn-in period and use this fixed value in KAO.
\subsection{Comparison with BOA}
We start this Section with a short comparison with the BOA algorithm of \cite{wintenberger2017optimal} that achieves similar regret bounds than the one obtained here. Theorem 4.2 in \cite{wintenberger2017optimal} shows that BOA, which is an algorithm based on surrogate losses, has nice generalization properties that extend the regret bounds in the adversarial setting into similar regret bounds in the stochastic adversarial setting. The price to pay for the generalization is a factor $2$ in the regret bounds. We show that this factor $2$ is avoidable under assumption {\bf (H)} with an algorithm such as KAO which uses the surrogate risks rather than the surrogate losses. Finally, notice that the use of the risk allows getting a.s. regret bounds in the well-specified stochastic unbounded setting rather than high-probability regret bounds only in bounded settings.

\subsection{The static iid setting}
In the iid setting, we consider an aggregation of static Kalman recursions with $P_0^{(m)}=1/\lambda^{(m)} I$, $\lambda^{(m)}>0$ which coincides with online ridge regression starting at $\hat \theta_0^{(m)}$. A natural estimator for $\sigma^2$ is the mean of the mean square residuals $M^{-1}\sum_{m=1}^M \hat \sigma_t^{(m)}$.  This setting is very specific since $\mu_t=\E_{t-1}[y_t]= \E[y_t]= X_t^\top \theta_t^{(m)}= X_t^\top \theta^{(m)}= X_t^\top \theta^\ast$ for any $1\le m\le M$ under {\bf (H)} and some fixed $\theta^\ast\in \Theta$ corresponds to the well specified setting.  Consider for a moment $D=\max_{t\ge 1}\max_{1\le m\le M}|X_t^\top (\theta^* -\hat \theta_t^{(m)})|$ as random.   Moreover one can estimate $G^{(m)}=4D^2$ such that, applying KAO with adaptive multiple learning rates we obtain the model selection regret bound
$$
\sum_{s=1}^t L_s(\hat y_s)\le \min_{1\le m\le M}\Big(\sum_{s=1}^t L_s(\hat y_s^{(m)})+O(- D^2\log \tilde \rho_0^{(m)}))\Big)\,.
$$
It is interesting to combine this bound with the regret bounds on the ridge regression when the design $(X_t)$ is iid, bounded by $X$ and such that $\E[X_tX_t^\top]$ has a positive lowest eigenvalue $\Lambda_{min}$. Applying Theorem 14 of \cite{de2020stochastic}, the $m$th Kalman recursion achieves for any $\theta\in\R^d$ and any $t\ge 1$
$$
\sum_{s=1}^t  L_s(\hat y_s^{(m)}) \le \sum_{s=1}^tL_s(X_s^\top \theta)+O\Big({\lambda^{(m)}}^3  \|\theta-\hat \theta_0^{(m)}\|_2^6+d \log\Big(\frac t{\lambda^{(m)}}\Big) +\log(\delta^{-1})^3 \Big)\,,
$$
with probability at least $1-\delta$. Moreover, the localization strategy of \cite{de2020stochastic} shows that   under the same probability $D$ can be considered as a constant.
Then KAO achieves the regret bound in expectation, valid for any $\theta\in \R^d$ and any $1\le m\le M$,
\begin{multline*}
\sum_{s=1}^t  L_s(\hat y_s) \le \sum_{s=1}^t L_s(X_s^\top \theta) +O\Big({\lambda^{(m)}}^3  \|\theta-\hat \theta_0^{(m)}\|_2^6+d \log\Big(\frac t{\lambda^{(m)}}\Big)  \\
- D^2\log \tilde \rho_0^{(m)} +\log(\delta^{-1})^3 \Big) \,,
\end{multline*}
with probability $1-M\delta$.
Aggregation can be seen as an online alternative of cross-validation for tuning the starting point of the ridge regression algorithm and the regularization parameter.  \\

As an illustration one should consider $\hat \theta_0^{(m)}$ may be taken equal to $\alpha (e_i)_{1\le i\le d}$ where $(e_i)_{1\le i\le d}$ is the canonical basis and $\alpha$ takes value on $[-d,d]\cap \Z$. Moreover $\lambda^{(m)}$ should be taken on an exponential $d$ finite grid of $(0,\infty)$. The number of Kalman recursions is $M=O(d)$  and choosing uniform weights yields to a regret for any $\lambda>0$ on the grid, any $1\le i\le d$ and any $\alpha \in [-d,d]\cap \Z$ as
\begin{multline*}
\sum_{s=1}^t L_s(\hat y_s)\le \sum_{s=1}^t L_s(X_s^\top \theta)+O\Big(\lambda^3  \|\theta-\alpha e_i\|_2^6+d \log\Big(\frac t{\lambda }\Big)  \\
+ D^2\log  d+\log(\delta^{-1})^3 \Big) \,,
\end{multline*}
with probability $1-d\delta$.
Other aggregation strategies on least-squares estimators are described in \cite{leung2006information}. Restrictions of our framework are the well-specification condition {\bf (H)} and the presence of the large constant $D^2$ in the model selection bound. One clear advantage is an explicit online procedure whereas least square estimators require the inversion of inverse matrices at each batch step.

\subsection{The dynamic setting}

In the dynamic setting, we consider that $(y_t)$ behaves as a centered random walk conditionally on the design. The Kalman recursions track the trajectory of the linear coefficients associated to the explanatory variables. Assume that the design is standardized such that $\E[{X_{t}^{(m)}}^2]= \E[{X_{t}^{(m')}}^2]$ for any $1\le m,m'  \le d$. Consider $M=d$ univariate Kalman recursions $d_m=1$ with $K^{(m)}=Q^{(m)}=1$. Then the random coefficients  $\theta^{(m)}_t$ satisfies the relation \eqref{eq:rw} and constitutes a random walk. If there exist $D,X>0$ satisfying $
D= \max_{t\ge 1}|\hat{y}_t-\mu_t|$ and $|X_{t,m}|\le X$ then one can bound, with high probability 
\begin{align*}
\max_{1\le t\le T}\max_{1\le m\le d}|\mathscr{L}_t^{(m)}|&\le 2DX \max_{1\le t\le T}\sum_{m=1}^d|\hat \theta^{(m)}_t|\\
&\le CDX\left(\sum_{ t=1}^T\left(\sum_{m=1}^d\E[(\hat \theta_t^{(m)})^2]^{1/2}\right)^{2}\right)^{1/2}
\end{align*}
for some high constant $C>0$. Then we can apply the result of \cite{guo1994stability} asserting that 
$\E[(\hat \theta_t^{(m)}-\theta_t^{(m)})^2]^{1/2}\le E$ for some $E>0$. Together with the fact that $\var( \theta_t^{(m)})=t$ by definition we obtain
$$
\max_{1\le t\le T}\max_{1\le m\le d}|\mathscr{L}_t^{(m)}|\le CDXd T\,.
$$
Then applying KAO with adaptive learning rate and doubling trick as in Remark \ref{rem:dt} with $G=CDXd T$,  we obtain with high probability the aggregation regret bound
$$
\sum_{s=1}^T L_s(\hat y_s)\le \sum_{s=1}^T  L_s\left(\sum_{i=1}^M\pi^{(m)}\hat y_s^{(m)}\right) + O(DX T^{3/2} d\log d) \,.
$$
This super-linear rate is due to the high fluctuations of the Kalman recursions when they track random walks $(\theta_t^{(m)})$. The Kalman recursions inherit the high variability of the random walks which is responsible for the high variability of the gradient and large $G=O(T)$. However, due to the unboundedness of the response, none of the existing regret bounds seem to apply in this setting.

\subsection{The expert aggregation setting}

The setting is similar to the previous one as $K^{(m)}$ is a diagonal matrix with non-null coefficients equals to $1$. Thus one has to assume the boundedness of the gradients to get a $\sqrt T$ regret for the aggregation problem. It is the usual assumption in the setting of aggregation of experts and then the regret is essentially divided by a factor $2$ compared with the regret bound obtained for BOA in \cite{wintenberger2017optimal} under the boundedness of the response. 
%However KAO for expert aggregation should be accompanied with a study of the variability of the Kalman recursions involved in the aggregation. Indeed, in the case the experts are not sufficiently stable, the Kalman recursions is unstable too and mimic the behavior of Kalman recursions tracking random walks; The boundedness assumption on the gradients is very unlikely and the regret bound is deteriorated as in the dynamic setting.
 It is worth mentioning again that the boundedness of the gradients of the conditional risk does not imply the boundedness of the response.  
\section{Simulation study}
\begin{figure}[!h] 
\centering
\subfigure{ \label{fig:kalman1}
\includegraphics[height=0.30\textheight,trim=0 0 0 0,clip]{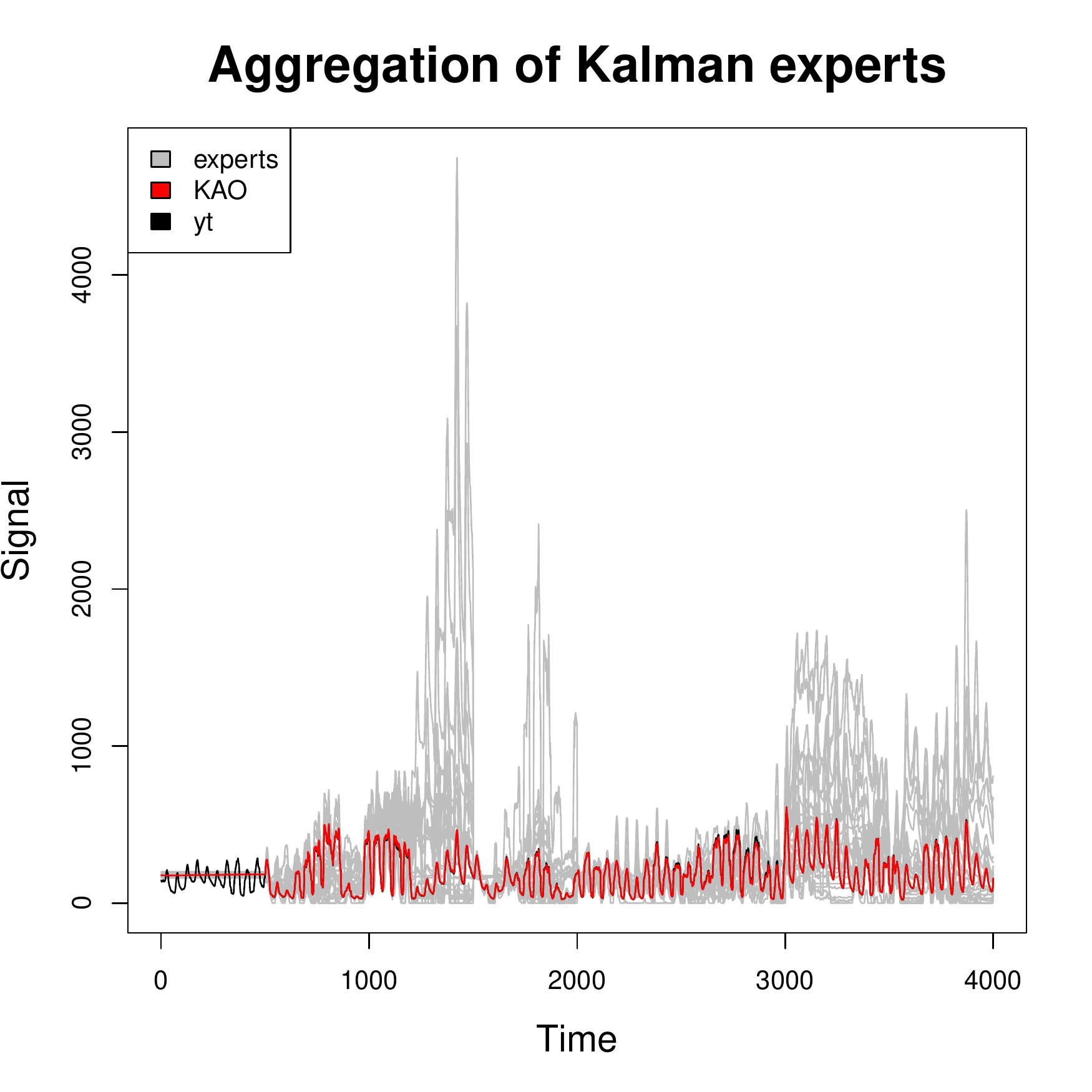} }
\quad 
\subfigure{\label{fig:cumulError1}
\includegraphics[height=0.30\textheight,trim=0 0 0 0,clip]{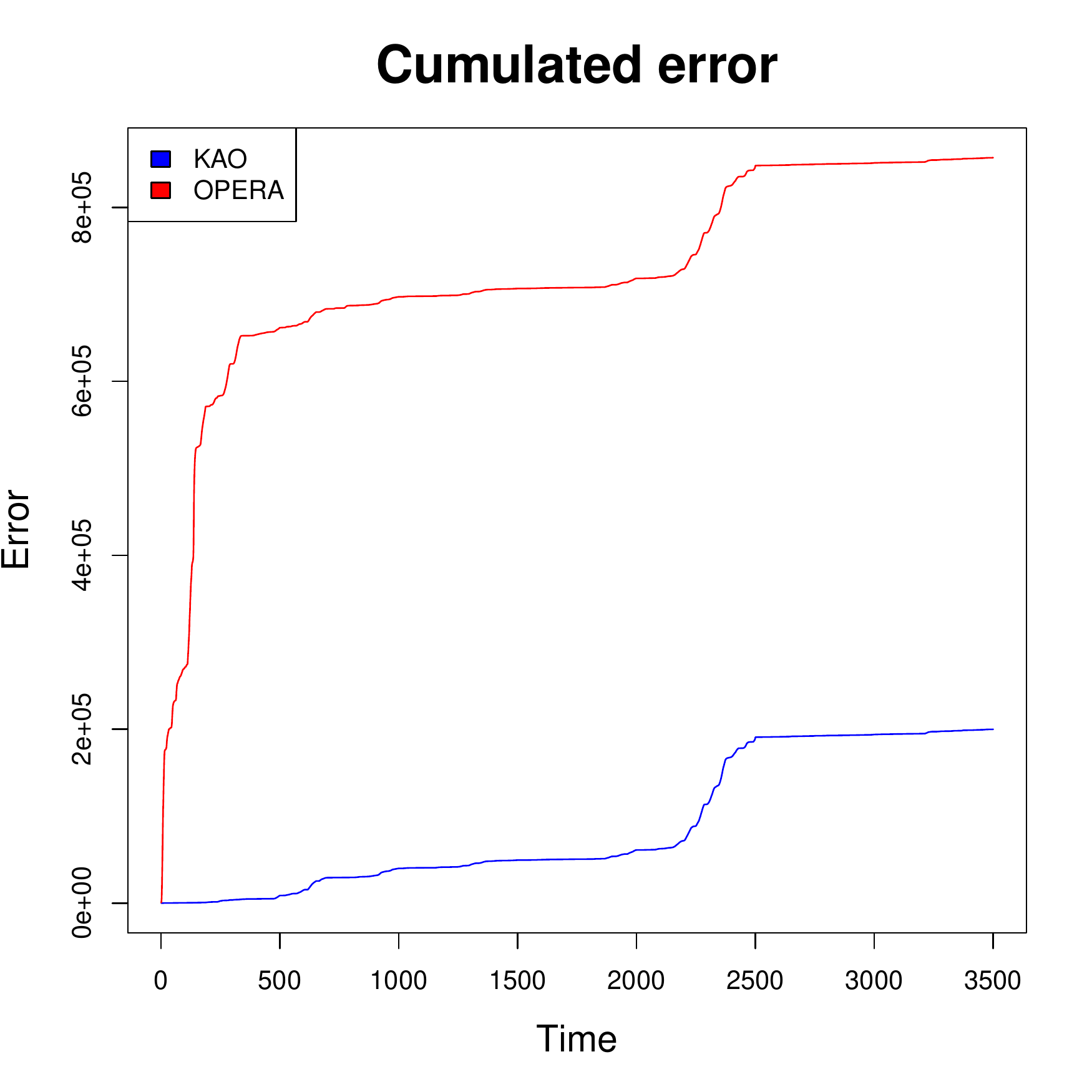} }
\caption{\textbf{One hour ahead prediction of $y_t$ using KAO, and cumulated prediction errors for KAO and OPERA}. The left panel shows one hour ahead prediction of $y_t$ using KAO and $\eta$ within a grid of values. The value of $\eta$ that minimizes the MSE is utilized to perform the prediction. These predictions are done in the case where the oracle is the best expert. The right panel shows the cumulated prediction errors for KAO and OPERA using the Kalman experts. the blue line represents the Kalman aggregation error, and the red one represents the error of the aggregation coming from the opera package. These predictions are done in the case where the oracle is the best expert\label{fig:kalman1_error1}.} 
\end{figure}
In this simulation study, we use some of the variables contained in the downloadable data set on the website of the RTE company (french TSO) that describes the hourly electricity consumption and production per type of production units in France from 2013 to 2017. We chose to simulate synthetic data from these true ones to be closer to a real application but controlling the true model at the same time. We generate synthetic data from a subset of these variables: the temperature, the gas production, the fuel production, the charcoal production, and the nebulosity. The square of the temperature and the cubic of the gas are jointly utilized as predictors in $X_t$ to simulate, under a state-space model, the signal $y_t$ that represents the electricity consumption. All the covariates are normalized to be in  $[0,1]$ by dividing each of them by their maximum value. The true model (that generates the true or the best expert) is a state-space model using the square of the temperature and the cubic of the gas as covariates in $X_t$ and Gaussian noise. Regarding the parameters of this state-space model, $\sigma=1.5$, $Q$ is of values $1$ on the diagonal and $0.9$ otherwise, $\theta_0$ is generated according to a gaussian law of mean 500 and covariance matrix identity, and $K$ is the identity matrix. We also compute 27 other Kalman experts using other combinations of covariates that are different from those used for getting the true (or best) expert.

\begin{figure}[!h] 
\centering
\subfigure{ \label{fig:kalman1}
\includegraphics[height=0.30\textheight,trim=0 0 0 0,clip]{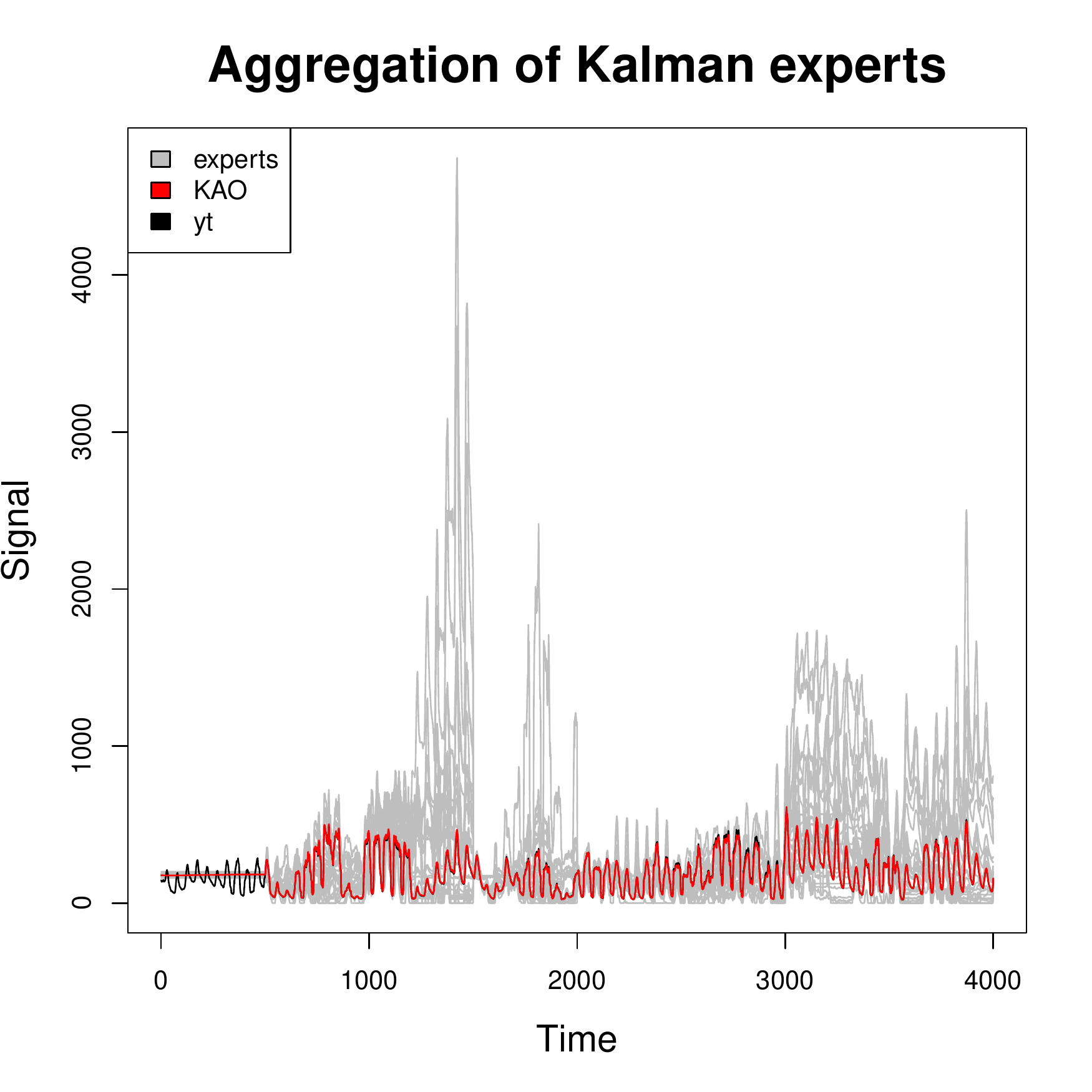} }
\quad 
\subfigure{\label{fig:cumulError1}
\includegraphics[height=0.30\textheight,trim=0 0 0 0,clip]{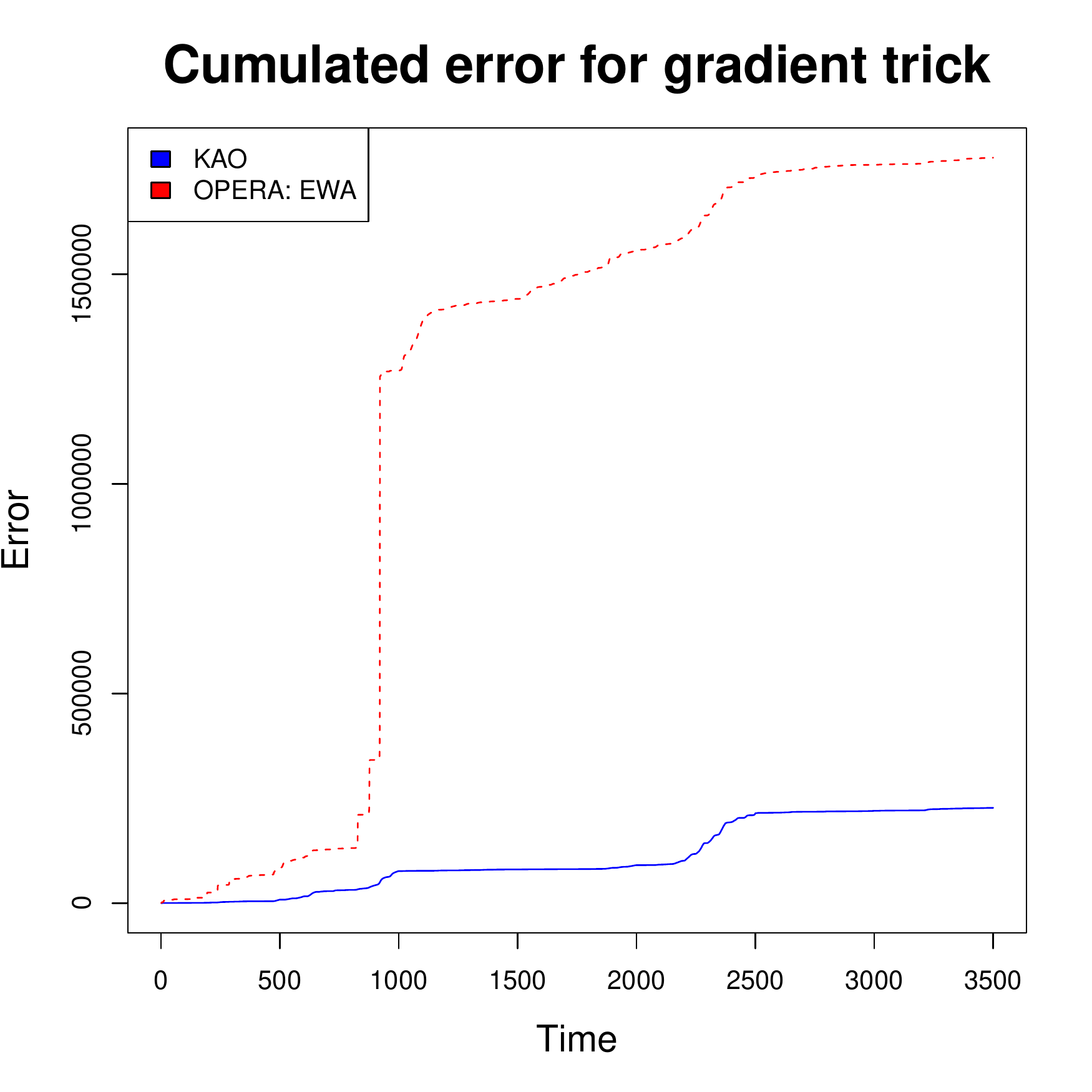} }
\caption{\textbf{One hour ahead prediction of $y_t$ using KAO, and cumulated prediction errors for KAO and OPERA}. The left panel shows one hour ahead prediction of $y_t$ using KAO and $\eta$ within a grid of values. The value of $\eta$ that minimizes the MSE is utilized to perform the prediction. These predictions are done in the case where the oracle is the best expert. The right panel shows the cumulated prediction errors for KAO and OPERA using the Kalman experts. the blue line represents the Kalman aggregation error, and the red one represents the error of the aggregation coming from the opera package. These predictions are done in the case where the oracle is the best convex combination of the Kalman experts\label{fig:kalman2_error2}.} 
\end{figure}

Each Kalman expert is computed in the sequential way as follows. We begin by fitting the model using the first observations ($y_1,\dots,y_{\mathrm{window}}$) contained in a window. Then the fitted model is utilized to predict the observations contained in a window ahead (i.e., $y_{\mathrm{window}+1},\dots,y_{2\mathrm{window}}$). At the $p$th step we use the observations $y_1,\dots,y_{p\mathrm{window}}$ to fit the model that is utilized to predict $y_{p\mathrm{window}+1},\dots,y_{(p+1)\mathrm{window}}$. We chose $\mathrm{window}=500$ as a good trade-off between a correct number of observations to estimate the state-space models and a good adaptation to changes. The prediction resulting from this procedure is called the Kalman expert and we, therefore, have $28$ Kalman experts.

 Simulations are done under the R software~\citep{r2019} and the predictive performance of the Kalman experts aggregation using KAO is compared with the aggregation performed using the R package \textit{opera}~\citep{opera2016} and the aggregation procedures therein. The aggregation obtained from the package opera is named OPERA when we are competing with the best expert and do not want to mention any specific aggregation procedure. We make one hour ahead prediction using KAO on the $28$ Kalman experts. In the case where the oracle is the best Kalman expert, the resulting prediction is plotted by a red curve in Figure~\ref{fig:kalman1_error1} at the left panel, where the signal $y_t$ is plotted by a black curve, and the experts are plotted using the gray color. We can see that the red line tracks well the black one, meaning that the aggregation from KAO performs well its prediction. More precisely, the MSE of KAO is $66.507$ which is approximately equal to the MSE of the best Kalman expert ($66.503$), and the MSE of OPERA is $253.06$. The right panel of Figure~\ref{fig:kalman1_error1} shows the cumulated error of KAO (in blue color) and OPERA (in red color). 
 We can see that KAO performs better than OPERA. Though both KAO and OPERA (precisely, EWA or BOA procedure) are based on exponential weights, the difference seen in their respective cumulated errors can be explained by the fact that KAO takes into account the underlying models that provide the experts, and OPERA doesn't have this information.

In the case where the oracle is the best convex combination of the Kalman experts, the one hour ahead predictions of $y_t$, using KAO, are plotted in the left panel of Figure~\ref{fig:kalman2_error2} in red color and the Kalman experts are plotted in gray color. We can also see that KAO tracks well the signal $y_t$ that is plotted in black color. Here, the MSE of KAO is $65.02$ against $223.37$ for OPERA, using the procedure BOA~\citep{wintenberger2017optimal}. The corresponding cumulated errors are plotted in the right panel for KAO (in blue color) and OPERA (in red color). The curves of the cumulated errors show that KAO has a better predictive performance than OPERA.

\begin{figure}[!h] 
\centering
\subfigure{ \label{fig:mse1}
\includegraphics[height=0.30\textheight,trim=0 0 0 0,clip]{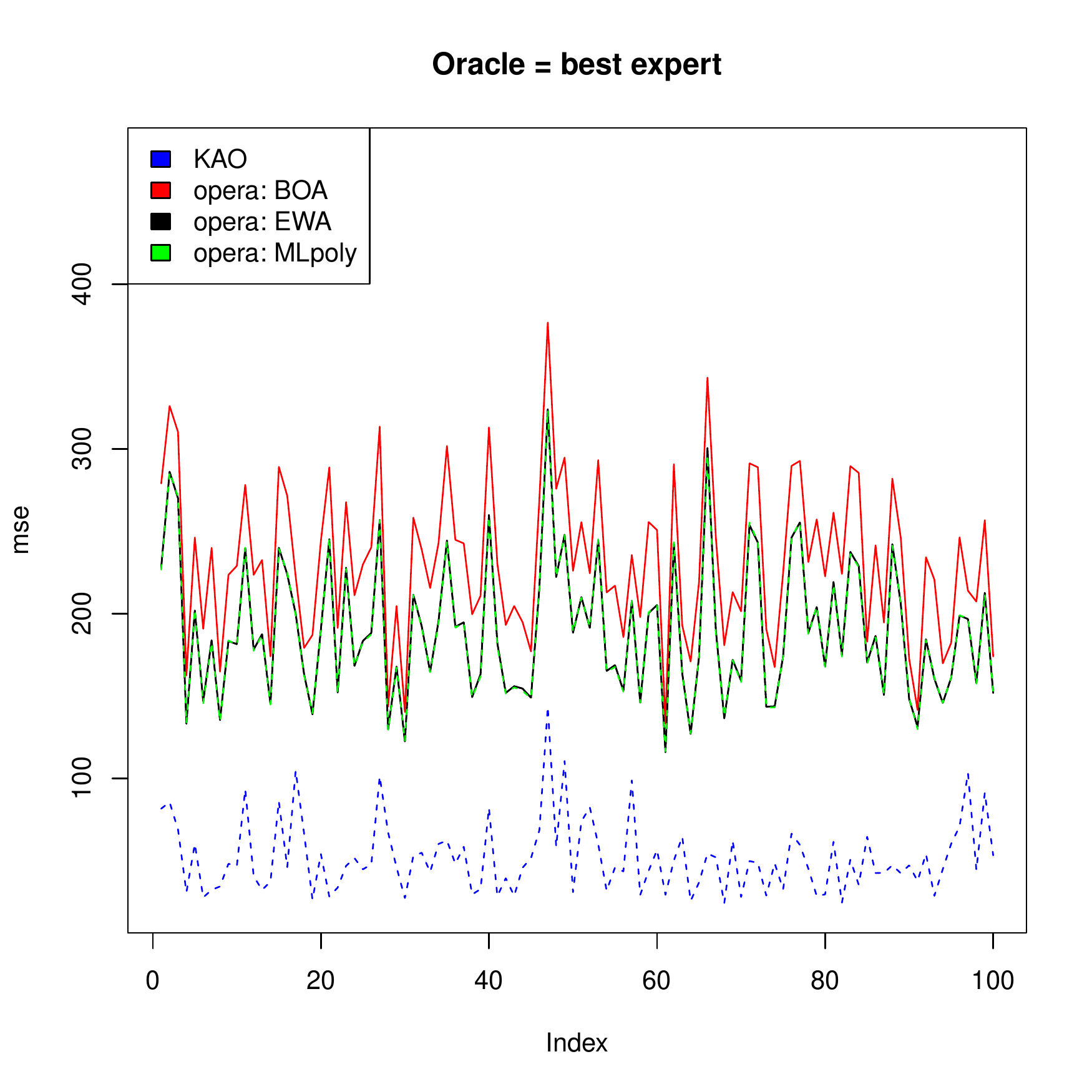} }
\quad 
\subfigure{\label{fig:mse2}
\includegraphics[height=0.30\textheight,trim=0 0 0 0,clip]{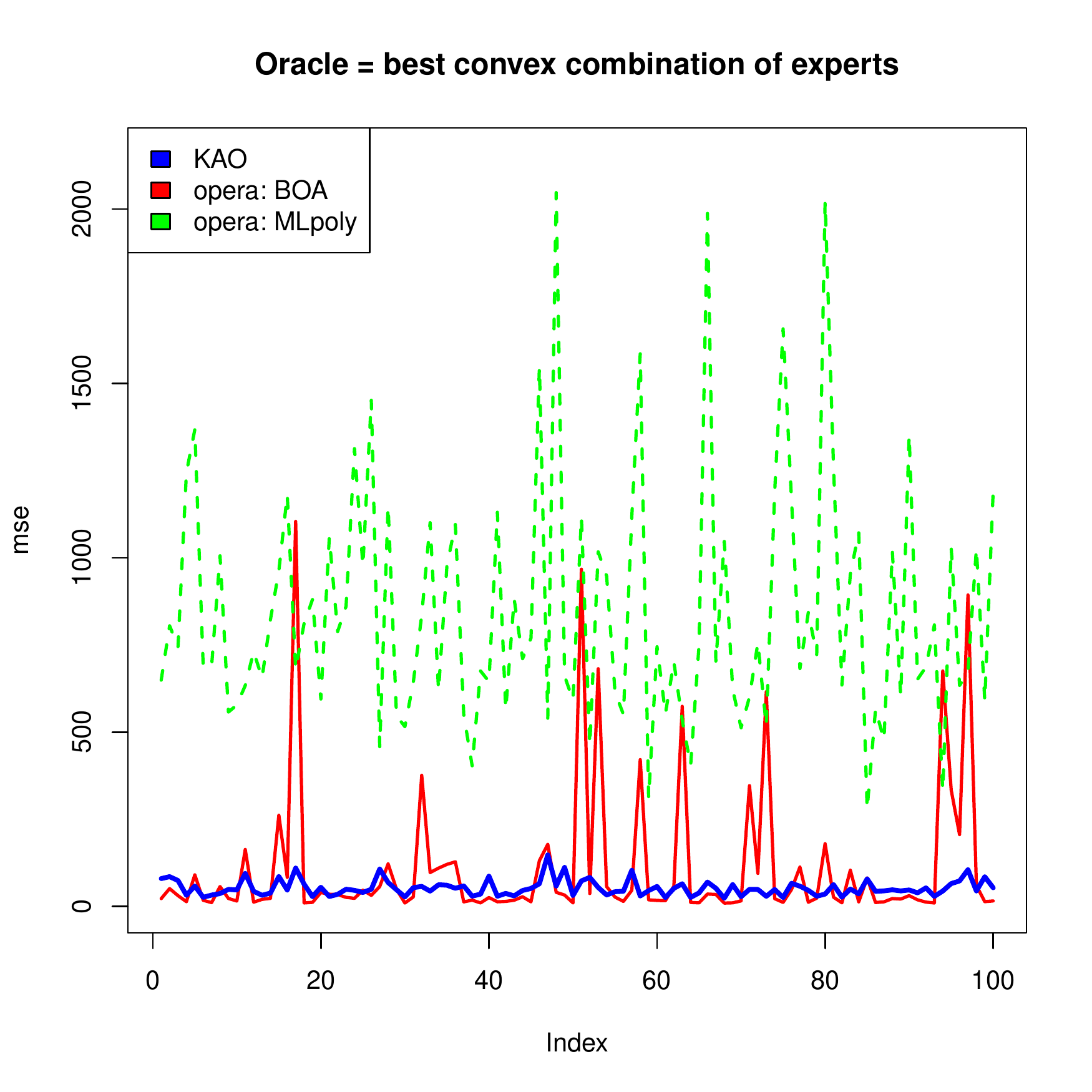} }
\caption{\textbf{MSE of $100$ aggregations of Kalman experts  using KAO and the procedures contained in opera.} The left panel shows the curves of computed mse, where each aggregation procedure competes with the best expert. KAO is plotted in blue color, BOA in red color, EWA (Exponentially Weighted Average ) in black and MLpoly~\citep{gaillard2014second} in green color. The right panel shows the curves of the mse computed when the aggregation competes with the best convex combination of experts. KAO is in blue color, BOA in red color and MLpoly in green color.}
\label{fig:mse}
\end{figure}

We simulate $100$ collections of Kalman experts corresponding to $100$ simulated datasets. Each collection of Kalman experts contains $28$ different experts. We then perform the aggregation of each collection of Kalman experts using KAO and the procedures within OPERA for each type of oracle. The MSE of the aggregations are computed and plotted in Figure~\ref{fig:mse}. The left panel (Figure~\ref{fig:mse1}) shows the curves of the MSE of the aggregations performed in the case where the oracle is the best expert. KAO (dashed blue curve) presents the lowest MSE within all the aggregation procedures, followed by MLpoly and EWA. The right panel (Figure~\ref{fig:mse2}) shows the aggregations' MSE in the case where the oracle is the best convex combination of the experts. We can see that KAO (blue curve) has not only the best MSE but also presents more stability than BOA (red curve) and MLpoly (dashed green curve). Here, reversely to the case where the aggregation competes with the best expert, BOA is better than MLpoly. This simulation study seems to point out that it may be worth of interest to take into account the underlying model that generates the experts when aggregating them.

\begin{table}[h!]
 \scriptsize
\caption{{\bf  Root Mean Square Error square of different aggregation procedures (relative to RMSE of the best convex combination). Kalman Experts.}}
 \begin{center}
\begin{tabular}[c]{ccc}
 \toprule
\bf Procedure & \bf rmse (with GT) & \bf rmse (without GT) \\
 \midrule 
 \multirow{1}{2cm}{Best expert}
& 1.15 & 1.15 \\ \\
\multirow{1}{2cm}{Uniform}
& 1.11 & 1.11 \\ \\
\multirow{1}{2cm}{MLpoly}
& 1.06 & 1.16 \\ \\
\multirow{1}{2cm}{BOA}
& 1.07 & 1.11 \\ \\
\multirow{1}{2cm}{KAO}
& 1.05 & 1.07 \\ \\
\multirow{1}{2cm}{Best convex}
& 1 & 1 \\
\bottomrule
\end{tabular}
 \end{center}
 \label{tab:rmse1}
\end{table} 
\section{Application}
In this section we apply the KAO algorithm to aggregate ten experts $f_{m,t}, 1\le m\le M\ $  that are meant to predict the daily electricity consumption in France (see \cite{Ba12},  \cite{design_expert2014} for previous work on french load data) at times $t \in (1,2,...,T)$. These experts are provided by different models that are black boxes. Thus, we consider the expert setting previously defined in \ref{subsec:kalman_rec}. For each expert we stack their predictions $f_{m,t}\in \R$ in $X_t^{(m)}$ together with the intercept and the past error $e_{m,t-1}=(y_{t-1}-f_{m,t-1})$, i.e.,
$$
X_t^{(m)}=(1,f_{m,t},e_{m,t-1}), \qquad t\ge 1\,.
$$

and each state-space model $m$ is defined by the state equation:

$$
\theta_{t}^{(m)} =  \theta_{t-1}^{(m)} +z_t^{(m)} \,,\qquad t\ge 1.
$$

the covariance matrices  $Q^{(m)}, 1\le m\le M\ $ and the variance of the noise $\sigma^{2(m)}$ are estimated using an EM algorithm on the first half of the data ($t \in (1,2,...,T/2)$) and we use the second half to evaluate KAO performances and compare it to other aggregation rules.

\begin{table}
 \scriptsize
\caption{{\bf Root Mean Square Error square of different aggregation procedures (relative to RMSE of the best convex combination). AR experts.}}
 \begin{center}
\begin{tabular}[c]{ccc}
 \toprule
\bf Procedure & \bf rmse (with GT) & \bf rmse (without GT) \\
 \midrule 
  \multirow{1}{2cm}{Best expert}
& 1.18 & 1.18 \\ \\
\multirow{1}{2cm}{Uniform}
& 1.11 & 1.11 \\ \\
\multirow{1}{2cm}{MLpoly}
& 1.07 & 1.19 \\ \\
\multirow{1}{2cm}{BOA}
& 1.07 & 1.09 \\ \\
\multirow{1}{2cm}{Best convex}
& 1 & 1 \\
\bottomrule
\end{tabular}
 \end{center}
 \label{tab:rmse2}
\end{table} 
The predictive risk of $\hat{y}^{(m)}_t=X_t^{(m)}\theta_{t}^{(m)}$ is used for computing the loss and the pseudo-loss that are needed to perform KAO. The aggregation performance of KAO (on these experts) is compared with that of both MLpoly and BOA that are two aggregation procedures available in the opera package. The results are contained in Table~\ref{tab:rmse1} where GT means Gradient Trick. GT, therefore, refers to the case where the oracle of the aggregation procedure is the experts' best convex combination. For confidentiality reasons, errors are expressed relatively to the RMSE of the best convex combination.
\begin{figure}[h!] 
\centering
\subfigure{ \label{fig:kao_weights}
\includegraphics[height=0.30\textheight,trim=0 0 0 0,clip]{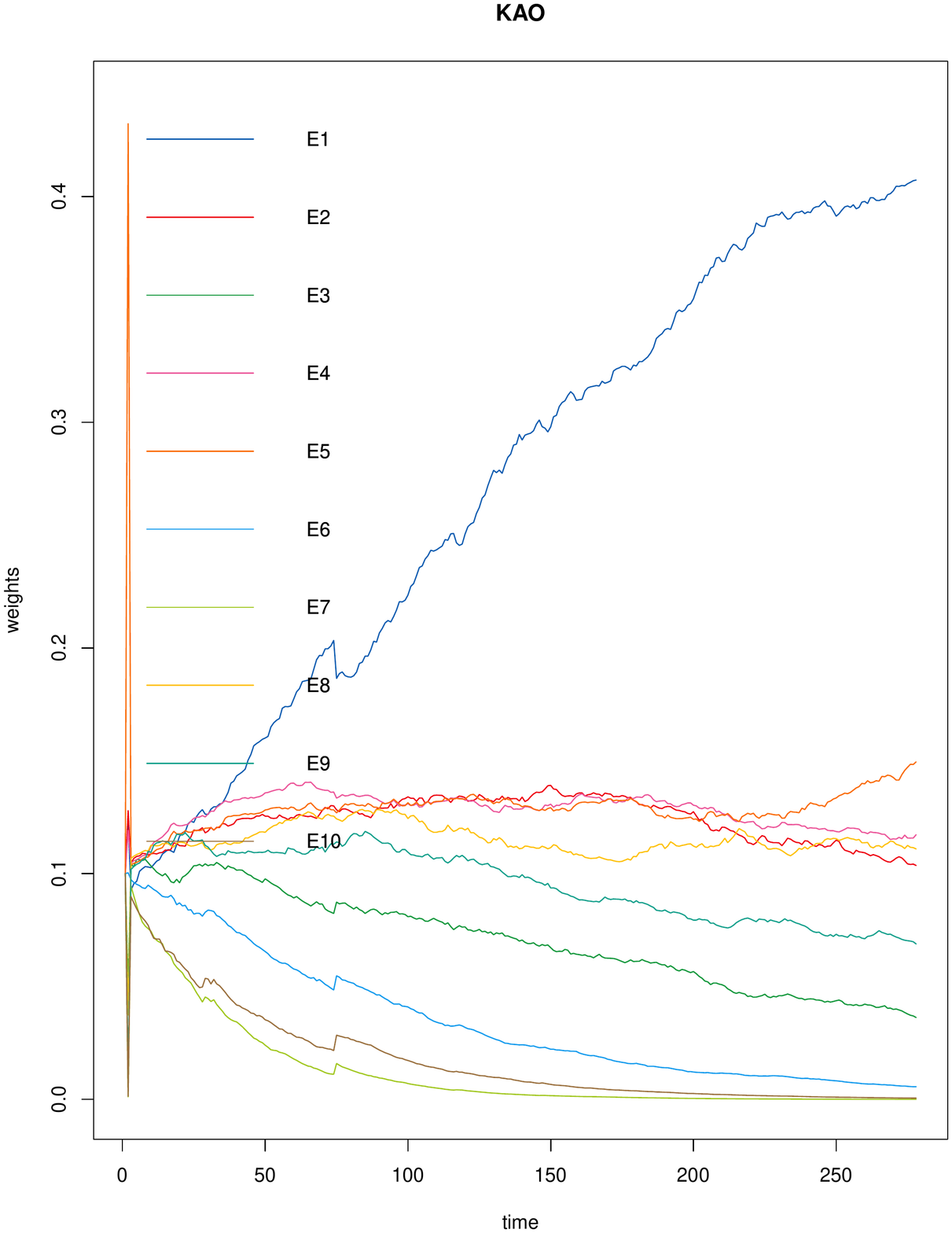} }
\quad 
\subfigure{\label{fig:mlpol_weights}
\includegraphics[height=0.30\textheight,trim=0 0 0 0,clip]{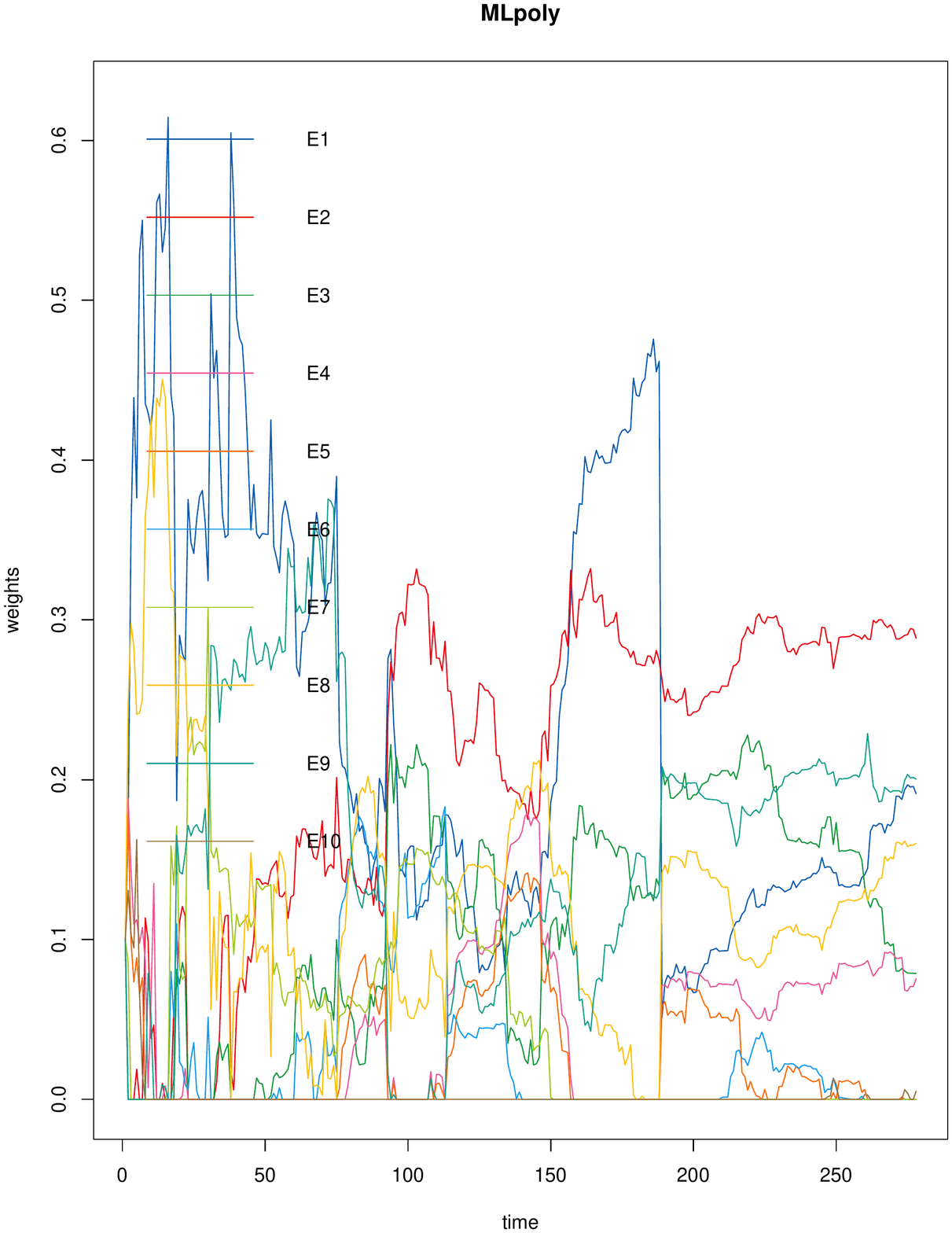} }
\caption{\textbf{Experts weight according to KAO and MLpoly, using the gradient trick}. The left panel shows the weights assigned to the corrected experts by KAO where the oracle is the best experts combination. The right panel shows the weights assigned by MLpoly, using the gradient trick. The experts are denoted by $\mathrm{E1},\dots,\mathrm{E10}$.} 
\label{fig:weights}
\end{figure}

The \textit{uniform} procedure is the experts mean, and the procedure \textit{best convex} is indeed the experts' best convex combination. All of these procedures are performed on the corrected experts. We clearly see that KAO performs slightly better (rmse $=1.05$ with GT and rmse $=1.07$ without GT) than both MLpoly (rmse $=1.06$ with GT and rmse $=1.16$ without GT) and BOA (rmse $=1.07$ with GT and rmse $=1.11$ without GT). In order to check if the Kalman correction is worth of interest, we make a direct autoregressive correction of the experts that are then aggregated using MLpoly and BOA (not KAO as we need an estimate of the risk for that). The results are contained in Table~\ref{tab:rmse2} and show that all the procedures are less accurate when the Kalman correction is not applied. 

The weights that are assigned to the corrected experts (by KAO and MLpoly) are plotted in Figure~\ref{fig:weights}. The weights coming from KAO are more smooth (see Figure~\ref{fig:kao_weights}) than those provided by MLpoly (see Figure~\ref{fig:mlpol_weights}). This smoothness of KAO weights can be explained by the fact that the procedure uses the underlying properties of the model that provide the experts. This information is used to anticipate the forthcoming performance of each expert.

\section{Conclusion}
In this paper, we show that the prediction obtained by aggregating the predictions coming from a finite set of experts can be improved by taking into account the properties of the underlying models that provide the experts' prediction.
We place ourselves in the case where all the predictions provided by the experts come from fitting state-space models using Kalman recursions. By using exponential weights, two settings are considered: 1) the aggregation competes with the best expert (also considered as model selection), and 2) the aggregation competes with the best convex combination of the experts. We consider adaptive multiple learning rates in order to achieve the optimal rates in these two schemes for a unique procedure. The quality of the aggregation's prediction has been improved by taking advantage of the full knowledge of the Kalman experts, using their predictive risk in an unbounded well-specified setting. 
In the simulations studies, we notice a great recovery of stability of KAO (our aggregation procedure), where all other existing aggregation procedures may be sometime somewhat unstable, potentially due to the lack of boundedness of the responses. The aggregation procedure KAO is also applied to some existing experts coming from unknown models, where we suggest correcting the errors of the experts using Kalman recursions. This strategy allows for approximating the theoretical weights needed for KAO and shows a quite important increase in the accuracy of the aggregation. In the case where the errors of the experts show no stationary behavior (for example, when there exist some cluster of variance), it should be interesting to adapt the fitting of the underlying state-space model in order to remain accurate.

\section*{References}

\bibliography{\jobname}
%\bibliography{mybibfile}

\end{document}